\documentclass[]{fairmeta}

\usepackage[most]{tcolorbox}
\usepackage{algorithm}
\usepackage{algpseudocode}
\usepackage{booktabs}
\usepackage{graphicx}
\usepackage{subcaption}

\usepackage{mathtools} % gives \coloneqq, \eqqcolon, \vcentcolon
\usepackage{amsthm}
\usepackage{listings}
\usepackage{booktabs,multirow,graphicx}
\usepackage{hyperref}
\usepackage{wrapfig}
\usepackage{caption}
\usepackage[most]{tcolorbox}

\definecolor{TakeawayBack}{HTML}{EEF5FF}  % light blue fill
\definecolor{TakeawayFrame}{HTML}{7BA3D7} % medium blue border

\newtcolorbox{takeaway}[1][]{
  enhanced, breakable,
  colback=TakeawayBack,
  colframe=TakeawayFrame,
  boxrule=0.5pt, arc=2pt,
  left=8pt, right=8pt, top=6pt, bottom=6pt,
  before skip=0.8\baselineskip, after skip=0.8\baselineskip,
  coltitle=black, fonttitle=\bfseries,
  colbacktitle=TakeawayBack,
  attach title to upper,           % put title inside the box
  after title={\quad},             % space after title
  title={#1}                       % optional [title]
}

\numberwithin{equation}{section}

\theoremstyle{plain} % italic body, bold head
\newtheorem{theorem}{Theorem}[section]
\newtheorem{lemma}[theorem]{Lemma}
\newtheorem{corollary}[theorem]{Corollary}

\newtheorem{proposition}[theorem]{Proposition}

\theoremstyle{definition} % upright body
\newtheorem{definition}[theorem]{Definition}

\theoremstyle{remark} % upright body, italic head

\crefname{theorem}{theorem}{theorems}
\Crefname{theorem}{Theorem}{Theorems}
\crefname{lemma}{lemma}{lemmas}
\Crefname{lemma}{Lemma}{Lemmas}
\crefname{corollary}{corollary}{corollaries}
\Crefname{corollary}{Corollary}{Corollaries}
\crefname{proposition}{proposition}{propositions}
\Crefname{proposition}{Proposition}{Propositions}
\crefname{claim}{claim}{claims}
\Crefname{claim}{Claim}{Claims}
\crefname{definition}{definition}{definitions}
\Crefname{definition}{Definition}{Definitions}
\crefname{assumption}{assumption}{assumptions}
\Crefname{assumption}{Assumption}{Assumptions}
\crefname{observation}{observation}{observations}
\Crefname{observation}{Observation}{Observations}
\crefname{remark}{remark}{remarks}
\Crefname{remark}{Remark}{Remarks}

\usepackage{xcolor}

\usepackage{tikz}
\newcommand{\cblock}[3]{
  \hspace{-1.5mm}
  \begin{tikzpicture}
    [
    node/.style={square, minimum size=10mm, thick, line width=0pt},
    ]
    \node[fill={rgb,255:red,#1;green,#2;blue,#3}] () [] {};
  \end{tikzpicture}%
}
 
\usepackage{amsmath,amsfonts,bm}

\def\eqref#1{equation~\ref{#1}}
\def\1{\bm{1}}

\usepackage{url}
\usepackage{enumitem}
\usepackage[para,online,flushleft]{threeparttable}
\usepackage{subcaption} % provides subfigure (a)/(b) + shared caption
\usepackage{adjustbox}

\newcommand{\fullmethod}{query-only test-time-training} % rename later
\newcommand{\method}{query-only TTT} % rename later
\newcommand{\shortmethod}{qTTT}
\newcommand{\para}[1]{\vspace{0.5em}\noindent\textbf{#1}}

\usepackage{alltt}
\usepackage{amsmath,amssymb}
\usepackage{xcolor}
\usepackage{tcolorbox}
\usepackage{dashrule} % for the dashed separator

\usepackage{tcolorbox,xcolor}
\tcbset{panelcommon/.style={
  enhanced, boxrule=1.2pt, arc=10pt,
  left=10pt,right=10pt,top=8pt,bottom=8pt
}}
\newtcolorbox{bluebox}[1][]{panelcommon,
  colframe=blue!70!black, colback=blue!3!white, #1}
\newtcolorbox{redbox}[1][]{panelcommon,
  colframe=red!70!black,  colback=red!3!white,  #1}

\usepackage{listings}
\lstdefinelanguage{json}{
  basicstyle=\ttfamily\small,
  showstringspaces=false,
  breaklines=true,
  frame=single,
  morestring=[b]",
  stringstyle=\color{teal!60!black},
}

\title{Let's (not) just put things in Context:\\Test-Time Training for Long-Context LLMs}

\author[2\ 3\ *]{Rachit Bansal}
\author[4\ *]{Aston Zhang}
\author[5\ *]{\\Rishabh Tiwari}
\author[1]{Lovish Madaan}
\author[6\ *]{Sai Surya Duvvuri}
\author[6\ *]{Devvrit Khatri}
\author[1]{\\David Brandfonbrener}
\author[2\ 3]{David Alvarez-Melis}
\author[1]{Prajjwal Bhargava}
\author[1]{Mihir Sanjay Kale}
\author[2]{Samy Jelassi}

\affiliation[1]{Meta}
\affiliation[2]{Harvard University}
\affiliation[3]{Kempner Institute at Harvard}
\affiliation[4]{OpenAI}
\affiliation[5]{UC Berkeley}
\affiliation[6]{UT Austin}

\contribution[*]{Work done while at Meta}

\abstract{
\looseness=-1
Progress on training and architecture strategies
have enabled LLMs with millions of tokens in context length.
However, empirical evidence suggests that
such long-context LLMs can \emph{consume}
far more text than they can reliably \emph{use}.
On the other hand, it has been shown that % there is a growing interest in
inference-time compute can be used
to scale performance of LLMs,
often by generating thinking tokens,
on challenging tasks involving
multi-step reasoning.
Through controlled experiments on sandbox long-context
tasks, we find that such inference-time strategies
show rapid diminishing returns, and fail at long context.
We attribute these failures to \emph{score dilution},
a phenomenon inherent to static self-attention.
Further, we show that
current inference-time strategies
cannot  % satisfy margin requirements to
retrieve relevant long-context signals
under certain conditions.
We propose \emph{\fullmethod{}} (\shortmethod{})
that, through targeted gradients updates on the given context,
provably overcomes limitations of static self-attention.
We find that this simple shift in how
inference-time compute is spent
leads to consistently large performance improvements
across models and long-context benchmarks.
\shortmethod{} leads to massive 12.6\% and
14.1\% points improvements for Qwen3-4B on
average across subsets of
LongBench-v2 and ZeroScrolls benchmarks.
The takeaway is practical: for long context, a small amount of context‑specific training is a better use of inference compute than current inference-time scaling strategies like producing more thinking tokens.
}

\correspondence{\email{rachitbansal@g.harvard.edu}, \email{az@astonzhang.com}}
\begin{document}

\maketitle
\section{Introduction}
\label{sec:intro}

\looseness=-1
Many ambitious LLM use-cases are rooted in long context: analyzing scientific corpora~\citep{katz2023natural,taylor2022galactica}, synthesizing books~\citep{kryscinski2021booksum}, maintaining rich multi-turn histories~\citep{park2023generative,zhou2023webarena}, and reasoning over large multi-file code repositories~\citep{jimenez2024swebench,zhang2023repocoder}.
Recent progress in pre-training and architectural strategies
have enabled context windows with millions of tokens~\citep{yang_rope_2025, ding2402longrope, reid2024gemini,anthropic2024}.
% —promising to "just read everything."
In practice, however, persistent failure modes remain: models miss clauses buried in lengthy documents, overlook function definitions deep in repositories, or fail to retrieve facts from prior turns even when the relevant content is present ``in context''~\citep{liu2023lost,hsieh2024ruler,kamradt2024needle}.
% The gap between what models can read and what they can reliably use is now one of the main bottlenecks for high-stakes deployments~\citep{levy2024same,shaham2023zeroscrolls}.

\looseness=-1
Concurrently, there is a growing interest in using
inference-time compute to overcome limitations
of vanilla transformer models.
Methods such as chain-of-thought ``thinking" tokens
\citep{wei2023chain},
best-of-$n$ \citep{nakano2021webgpt,stiennon2020learning}, and other ``thinking" strategies \citep{zelikman2024quiet}
have shown promise.
% Current remedies tend to allocate \emph{more tokens} rather than \emph{better attention}. Chain-of-thought \citep{wei2023chain}, self-consistency \citep{wang2023selfconsistency}, best-of-$n$ \citep{nakano2021webgpt,stiennon2020learning}, and other ``thinking" strategies \citep{zelikman2024quiet} 
However, all these methods generate additional tokens
with the same static attention mechanism that is already under-allocating mass to the evidence.

\looseness=-1
We design two realistic sandbox tasks to perform
controlled experiments
and diagnose long-context failure modes.
We identify that % both, 
standard ``in-context only'' % and generating ``thinking tokens",
settings fail with growing context length (\autoref{fig:figure_1}).
We formalize this as a limitation of static,
finite-precision self-attention,
and term it \emph{score dilution}:
% with a growing number of distractor tokens, the logit for the target is insufficiently separated from those of distractors, reducing the softmax probability assigned to the target (\S\ref{sec:failures}).
In presence of ``distractor'' tokens,
logit on a ``target'' is insufficiently separated
from the distractor logits,
weakening the target probability mass. 
% We establish a \emph{logarithmic margin requirement}: As context length $T$ grows, in the worst case, the target–distractor logit margin must scale as $\Omega(\log T)$ to avoid vanishing target probability.
We establish that as context length $T$ grows,
% in the worst case,
the target–distractor logit margin
must scale as $\Omega(\log T)$ to avoid vanishing target probability.
We extend this analysis
to show that vanilla compute-scaling strategies, 
such as ``thinking'' tokens,
cannot retrieve the signal from buried target tokens.
% beyond an $\varepsilon$-fraction (\S\ref{sec:failures}).

\looseness=-1
Hence, a natural question arises:
\textit{How can we best use inference-time compute to improve long-context retrieval and reasoning?}
We revisit test-time training (TTT) \citep{liu_ttt_2021, hardt_test-time_2024, akyurek_surprising_2025} as a way to adapt the model to a given long-context input rather than produce more text from an unchanged model.
Our key idea, \emph{\method{}} (\emph{\shortmethod{}}), is a computationally frugal approach: 
% perform a single forward pass over the long input to cache keys/values, then take several span-sampled updates
% 
% We propose a different use of inference compute: move from ``\emph{in-context}" generation to ``\emph{in-weights}" adaptation via \textbf{query-only test-time training (qTTT)}.
% 
Perform a single prefill to cache keys and values,
followed by a few \textbf{lightweight gradient updates exclusively on the query projection matrices} in the attention layers,
keeping all other parameters fixed
and reusing the key-value cache (\autoref{fig:qttt}).
We show theoretically that this targeted adaptation directly increases the separation between target and distractor logits for the specific context at hand, counteracting the limitations of vanilla in-context learning.

\looseness=-1
% Our results show that reallocating the budget from thousands of thinking tokens to a handful of query updates provides a principled and practical path to better use of long contexts without changing the pre-training strategies, architecture, or data.
% We conduct experiments on more than 15 real-world datasets across the
% SCROLLS~\citep{shaham2023zeroscrolls} and LongBench-v2~\cite{bai2023longbench} benchmarks
% on Qwen3 and Llama3 models, ranging from $1.7$B to $8$B parameters.
% As shown in Figure~\ref{fig:figure_1}~(c), we find that
% spending inference-time compute via \method{}
% improves performance across model sizes
% and datasets. Further, \shortmethod{}
% consistently exceeds the gains via intermediate thinking tokens
% on FLOP-matched comparisons,
% showing that it provides a better way to spend
% inference time compute for long contexts.
% Indeed, 
% \rachit{to elaborate on results here}.
We perform evaluations on 15+ real-world datasets
from popular long-context benchmarks,
ZeroScrolls~\citep{shaham2023zeroscrolls} and LongBench-v2~\citep{bai2023longbench}, with Qwen3 models spanning $1.7$B–$8$B parameters.
We observe consistently large performance gains
across model sizes and datasets.
% \method{} substantially overcomes limitations of vanilla attention and inference-time scaling strategies.
Under FLOP-matched inference-time compute budgets, \shortmethod{} consistently surpasses standard inference-time thinking strategies (\autoref{fig:figure_1_c}) with more than $20\%$ improvements
on code comprehension, multi-document QA,
and other multi-hop reasoning tasks.
% showing that directing the same compute toward adapting attention yields a more effective use of inference-time resources for long contexts.
% \rachit{quote some numbers here too.}
Our results call for reallocating inference-time budget
from thousands of ``thinking'' tokens to a small number of query updates
for long-context retrieval and reasoning
without altering pre-training, architecture, or data.

\begin{figure*}[t]
\centering
\begin{subfigure}[t]{0.31\textwidth}
\centering
\vskip 0pt
    \hspace*{-.2cm}\includegraphics[height=3.3cm]{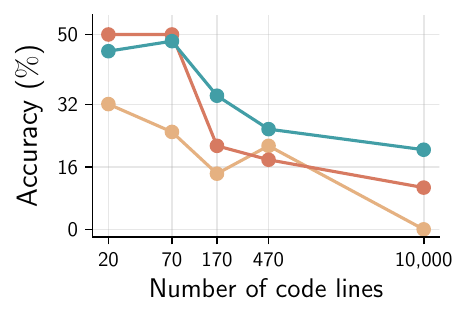}
    \vspace{.05cm}
    % \small 
    % %\centering
    % \mbox{\hspace*{3cm}Transformer:} %\cblock{15}{82}{186}\hspace{1mm}Alibi\hspace{1.5mm} \cblock{115}{194}{251}\hspace{1mm}HAlibi 
    % \\
    % \hspace*{3.7cm}GSSM: %\cblock{250}{128}{114}\hspace{1mm}LSTM\hspace{1.5mm}\cblock{184}{15}{10}\hspace{1mm}Mamba
    % \vspace{.6cm}
    \small
    % \mbox{\hspace*{0.5cm}\cblock{250}{128}{114} \hspace{0.5mm}In-Context Learning (ICL)}\\
    % \mbox{\hspace*{-.95cm}\cblock{184}{15}{10} \hspace{0.5mm}ICL + thinking}\\
    % \mbox{\hspace*{.5cm}\cblock{17}{52}{166}\hspace{1mm}Test-Time Training (qTTT)}
    {\hspace*{.5cm}\caption{Bug tracing in code repository}}
    \label{fig:in_distribution_copy}
\end{subfigure}
\hfill
\begin{subfigure}[t]{0.31\textwidth}%0.33
% \centering
\vskip 0pt
\hspace*{-.3cm}\includegraphics[height=3.3cm]{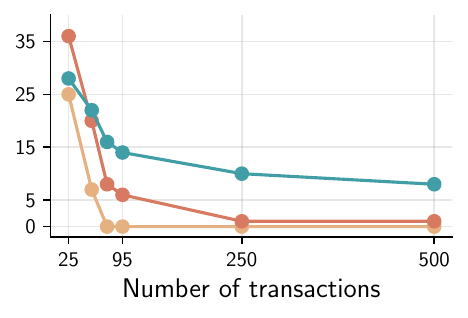}
\small
\mbox{\hspace*{-3.5cm}\cblock{229}{177}{129} \hspace{0.5mm}In-Context Only\hspace{3mm} \cblock{215}{122}{97} \hspace{0.5mm}With Thinking\hspace{3mm} \cblock{66}{158}{166}\hspace{1mm}With Query-only Test-Time Training (qTTT)}
\caption{Bug tracing in transaction logs}
\label{fig:length_gen}
\end{subfigure}
\hfill
\begin{subfigure}[t]
{0.31\textwidth}
\centering
\vskip 0pt
\includegraphics[height=3.2cm]{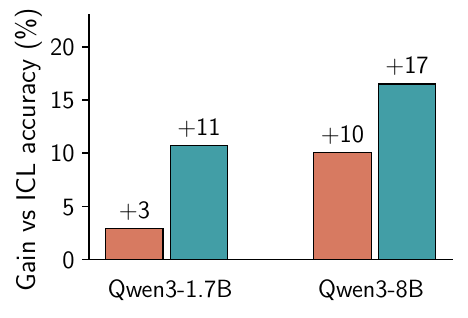}
\centering
\small
\vspace{0.55cm}
\caption{\label{fig:figure_1_c}LongBench-v2 + ZeroScrolls}
\end{subfigure}
\caption{\label{fig:figure_1}
% \looseness=-1
Query-only test-time training uses inference-time compute more effectively than ``thinking'' tokens for long contexts. \textbf{(a, b)} We construct two tasks to perform controlled long-context analysis: (a) bug localization in large code repositories, and (b) anomaly detection in transaction logs. As context length $T$ grows, in-context accuracy drops and thinking tokens show diminishing returns; with the same FLOP budget, \shortmethod{} consistently improves performance. \textbf{(c)} \shortmethod{} shows improvements across domains and model sizes on LongBench-v2 and ZeroScrolls benchmarks.
}
\end{figure*}

\para{Contributions.}
\begin{itemize}[leftmargin=*]
\item
% \textbf{Diagnosis of long-context failures.}
We construct sandbox tasks to demonstrate long-context failure modes (\S\ref{sec:empirical-limits}).
We formalize \emph{score dilution} in static, finite-precision self-attention and prove a \emph{logarithmic margin requirement}: the target–distractor logit gap must scale as $\Omega(\log T)$ to avoid vanishing target probability (\S\ref{subsec:theoretical-limits}). \looseness=-1
\item
% \textbf{Limits of decoding-based inference-time scaling.}
We show theoretically and empirically that current inference-time compute scaling strategies primarily scale decoding and cannot reliably meet the margin requirement; in particular, they cannot amplify the signal from buried targets beyond an $\varepsilon$-fraction (\S\ref{sec:failures}).
\item
% \textbf{Query-only test-time training.}
We introduce \method{} (\shortmethod{}): a compute-frugal TTT procedure that performs one prefill to cache K/V, then applies a few gradient updates \emph{only} to query projections while reusing the KV cache, directly increasing target–distractor separation (\S\ref{sec:method}).
% We provide FLOP-matched comparisons to decoding-based baselines for fair evaluation (\S\ref{subsec:flops-summary}).
\item
% \textbf{Broad empirical gains under fixed compute.}
On 15+ real-world datasets from ZeroScrolls and LongBench-v2, using Qwen3 models (1.7B–8B), \method{} consistently improves long-context performance and under FLOP-matched budgets, outperforms intermediate thinking-token baselines (\autoref{fig:figure_1_c}; \S\ref{sec:results}).
\end{itemize}

\looseness=-1
Since \shortmethod{} takes place at inference-time,
it can easily be applied on top of
other existing strategies for long-context modeling:
architectural changes such as sliding window attention~\citep{dai2019transformerxl,beltagy2020longformer}, adaptive positional encoding \citep{press2021alibi,su2024roformer}, training tweaks for longer windows \citep{chen2023extending,peng2023yarn},
or retrieval augmented generation~\citep{,borgeaud2022retro,izacard2022atlas}.

\section{Vanilla Compute-Scaling Strategies Fail for Long Contexts}
\label{sec:failures}

\looseness=-1
In this section, we analyze how increasing context length $T$ affects static quadratic-attention LLMs and common inference-time compute–scaling strategies. Using controlled synthetic tasks that mirror realistic long-context retrieval, we observe sharp performance degradation as $T$ grows, while generating intermediate ``thinking" tokens yields rapidly diminishing returns. We then provide a theoretical explanation: with static, finite-precision self-attention, the target logit suffers \emph{score dilution} as distractors accumulate, and avoiding this requires a \emph{logarithmic margin requirement}---the worst-case target–distractor logit gap must scale as $\Omega(\log T)$. Decoding-based inference strategies do not reliably meet this requirement; in contrast, small gradient-based adaptations can increase the margin, which motivates our methodology (developed in \S\ref{sec:method}). All proofs are provided in~\autoref{app:proofs-long-context}.

\subsection{Empirical Analysis on Synthetic Long-Context Tasks}
\label{sec:empirical-limits}

\looseness=-1
First, we empirically analyze the effect of context length
on vanilla transformer models
and current inference-time compute-scaling strategies.
% we perform controlled experiments on synthetic tasks.
% Particularly,
% we design two realistic controlled tasks that reflect
% % real-world use cases of long-context models while allowing us to
% where we can
% flexibly scale context lengths for our analysis.
% We begin by describing these tasks:
We study two synthetic retrieval tasks that
mirror realistic long-context
use cases while allowing control over the context length $T$.
For each example, the relevant evidence (``needle")
is held fixed and only the surrounding ``haystack" grows,
isolating the effect of length on retrieval.
We provide examples from our datasets in~\autoref{sec:dataset-examples}.

\para{Bug Localization in a Code Repository.}
Starting from a large open-source repository\footnote{
We use OLMo as a reference repository for the dataset: \url{https://github.com/allenai/OLMo}.
}, we inject a single-line logical bug and ask the model to identify and fix it. Examples of bugs include missing softmax temperature scaling in the attention mechanism and layernorm misplacement in the Transformer block (see Appendix for details). 
We vary the context length by the number of lines
$L$ exposed to the model.
% using a deterministic expansion procedure.
For a given bug instance,
we sample a span of $L$
lines around the bug,
extending to other files in the directory
for large $L$.
We create splits of the dataset
with $L$ ranging from $5$ to $10000$.
Across length conditions,
the bug location and content are held fixed;
only the surrounding code (the ``haystack”)
grows to introduce realistic,
semantically relevant distractors.

\para{Error in a Log of Transactions.}
We synthesize multi-account banking logs
with an initial state and a sequence of operations,
each line recording old$\rightarrow$new
balances and indexed with a \texttt{TX\_ID}.
Valid logs must satisfy invariants:
conservation of total funds,
non-negative balances, and arithmetic correctness.
We inject exactly one anomaly and consider the following bug types: \texttt{CALC\_ERROR} (incorrect arithmetic),
\texttt{NEGATIVE\_BAL} (over-debit),
\texttt{LOST\_UPDATE} (stale write overwrites a prior commit) and \texttt{DUPLICATE\_TXN} (same payment applied twice). 
The model must output the bug type and offending \texttt{TX\_ID}.
Context length is controlled by the number of operations $n$;
we sweep from $25$ operations to $500$ operations which varies the number of tokens from $\mathcal{O}(10^2)$ to $\mathcal{O}(10^4)$. 

\para{Findings.}
We evaluate Qwen3 models ranging from
$1.7$B to $8$B parameters on these synthetic tasks.
\autoref{fig:figure_1} shows the results
for the Qwen3-4B model.
For both tasks,
we see clear consistent trends:
(i) As the context lengths increases
(number of code lines/transaction logs),
the standard in-context performance
(i.e., without any additional inference-time compute)
decreases sharply.
(ii) Further, using inference-time compute
via thinking tokens improves performance
for shorter contexts,
but shows clear diminishing returns
as the context length increases,
asymptotically converging
close to the standard model performance for long contexts.

% \rachit{add a takeaway box. make the above paragraph stronger
% by clearly stating the conclusion from these experiments.
% provide a good segway to the theoretical section.}

\begin{takeaway}[Empirical Takeaway:]{Across both controlled tasks, holding the needle fixed and increasing the haystack length $T$ yields a sharp, monotonic drop in \emph{in-context} accuracy. Allocating inference-time budget to ``thinking" tokens offers only short-horizon gains with clear saturation at large $T$. These trends suggest a structural limitation of static attention in long contexts.}
\end{takeaway}

We now formalize this limitation as \emph{score dilution} and derive the resulting \emph{logarithmic margin requirement}, which explains why decoding-based scaling fails to recover retrieval (§\ref{subsec:theoretical-limits}).

\subsection{Preliminaries}
Recall, for a sequence of $T$ tokens with hidden representations $\{h_i\}_{i=1}^T \in \mathbb{R}^d$, each Transformer layer $\ell$ computes query, key, and value projections:
\begin{align}
    q_i^{(\ell)} &= W_Q^{(\ell)} h_i, \quad k_j^{(\ell)} = W_K^{(\ell)} h_j, \quad v_j^{(\ell)} = W_V^{(\ell)} h_j ,
\end{align}
where $W_Q^{(\ell)}, W_K^{(\ell)} \in \mathbb{R}^{d_k \times d}$ and $W_V^{(\ell)} \in \mathbb{R}^{d_v \times d}$ are learned projection matrices. 
Further, the scaled dot product between query $q_i$ and key $k_j$ gives the attention logits $z_{i,j}$ that are normalized via softmax to obtain attention weights $\alpha_{i,j}$. Finally, the output $o_i$ is a weighted sum of value vectors:
\begin{align}\label{eq:attn}
    z_{i,j} \coloneqq \frac{q_i^\top k_j}{\sqrt{d_k}},
\qquad
\alpha_{i,j} \coloneqq \frac{\exp(z_{i,j})}{\sum_{\ell=1}^T \exp(z_{i,\ell})},
\qquad
o_i=\sum_{j=1}^T \alpha_{i,j}v_j.
\end{align}
In the autoregressive setting, causal masking enforces $j \leq i$, so that each position $i$ can only aggregate information from its past. Multi-head attention extends this computation across several subspaces, allowing the model to capture diverse forms of dependency.

\para{In-Context Learning.}
This attention-based retrieval is the foundation of \emph{in-context learning} (ICL; \citep{dong2024surveyincontextlearning}).
By inserting task demonstrations, instructions, or relevant passages directly into the input, LLMs can adapt their outputs without parameter updates.
For applications such as analyzing codebases, synthesizing long documents, or sustaining multi-turn dialogues, the model must effectively identify and use information scattered across contexts of length $10^4$–$10^6$ tokens.

\para{Thinking Tokens.}
Given a prefix $x_{1:i}$ and a target at position $i{+}1$, \emph{thinking-token} methods~\citep{wei2022chainofthought,kojima2022large,wang2023selfconsistency} append $M\!\ge\!0$ auxiliary tokens at indices $t\in\{i{+}1,\dots,i{+}M\}$ before producing the final answer at $a\!=\!i{+}M{+}1$. Each token $t$ is generated with static parameters and the same attention kernel as in~\cref{eq:attn}, yielding logits $z_{t,j}$, weights $\alpha_{t,j}$, and outputs $o_t$ over the augmented sequence of length $T' \!=\! T{+}M$.

\looseness=-1
\begin{definition}[{Retrieval}]\label{def:retrieval}
\label{def:tau-retrieval}
When predicting token $x_{i+1}$, the relevant information may lie in a specific key–value pair $(k_{j^\ast}, v_{j^\ast})$ (the `\textit{needle}') at some earlier position $j^\ast < i$. For a threshold $\tau\in(0,1)$, we say that retrieval at position $i$ succeeds if $\alpha_{i,j^\star}\;\ge\;\tau.$ Equivalently, in margin form define
$\gamma_i \;\coloneqq\; z_{i,j^\star}\;-\;\log\!\sum_{j\neq j^\star} e^{z_{i,j}},$
then retrieval succeeds iff
\[
\gamma_i \;\ge\; \log\!\Big(\frac{\tau}{1-\tau}\Big).
\]
All other positions $j\neq j^\star$ are \emph{distractors}, contributing competing logits $\{z_{i,j}\}_{j\neq j^\star}$.
\end{definition}

\subsection{Theoretical Limitations of Static Attention and Thinking Tokens}
\label{subsec:theoretical-limits}

Informed by the empirical findings in~\S\ref{sec:empirical-limits},
we now analyze a single attention layer as in \cref{eq:attn} on the retrieval task (Definition~\ref{def:retrieval}).
We formalize the fundamental challenge of score dilution,
which arises when ``near-tie'' distractors inflate
the softmax denominator, causing even a unique maximal logit to receive vanishingly small attention mass.

\begin{lemma}[Score dilution]
\label{lem:score-dilution}
If at least $m$ distractor keys satisfy $z_{i,j}\ge z_{i,j^\star}-\Delta$ for some $\Delta\ge 0$, then
\[
\alpha_{i,j^\star}\;\le\; \frac{1}{1+m e^{-\Delta}}.
\]
In particular, if $m\ge cT$ for some $c>0$ and $\Delta=O(1)$, then $\alpha_{i,j^\star}\to 0$ as $T\to\infty$.
\end{lemma}
This lemma formalizes a simple intuition: When a constant fraction of tokens are within $O(1)$ logit of the needle, the attention budget cannot concentrate and the needle’s mass vanishes with $T$.

This dilution effect imposes a strict requirement on how much the target logit must stand out from all other distractors. The following corollary quantifies this necessary separation, showing that the required margin between needle and distractor must grow with the context length.
\begin{lemma}[Logarithmic margin requirement]
\label{lem:log-margin}
Fix $\varepsilon\in(0,1)$. If
\[
\min_{j\neq j^\star}\big(z_{i,j^\star}-z_{i,j}\big)\;\ge\;\log\!\Big(\frac{(T-1)(1-\varepsilon)}{\varepsilon}\Big),
\]
then $\alpha_{i,j^\star}\ge 1-\varepsilon$.
In particular, guaranteeing a fixed target mass against worst-case distractors requires a gap that scales as $\Omega(\log T)$.
\end{lemma}

\looseness=-1
Achieving a margin that scales logarithmically
is difficult for a model with static attention.
% A common strategy is to use inference-time compute
% to generate intermediate thinking tokens.
Next, we evaluate the strategy
of generating
thinking tokens
% efficacy of this approach
in satisfying % margin requirement of Lemma~\ref{lem:log-margin}.
the logarithmic margin requirement.

\begin{proposition}[Needle-signal bound for generated tokens]
\label{prop:needle-signal}
For any thinking token $t\in\{i{+}1,\dots,i{+}M\}$ and any $u\in\mathbb{R}^{d_v}$,
\[
\big\langle u,\, o_t \big\rangle
\;\le\;
\alpha_{t,j^\star}\,\big\langle u, v_{j^\star}\big\rangle
\;+\;
\big(1-\alpha_{t,j^\star}\big)\,\max_{j\neq j^\star}\big\langle u, v_j\big\rangle.
\]
\end{proposition}

\begin{corollary}[Specialization under small margin]
\label{cor:needle-signal-eps}
If the margin at token $t$ satisfies $\gamma_t \le \log\!\big(\varepsilon/(1-\varepsilon)\big)$ (equivalently, $\alpha_{t,j^\star}\le \varepsilon$ by Definition~\ref{def:retrieval}), then
\[
\big\langle u,\, o_t \big\rangle
\;\le\;
\varepsilon\,\big\langle u, v_{j^\star}\big\rangle
\;+\;
(1-\varepsilon)\,\max_{j\neq j^\star}\big\langle u, v_j\big\rangle.
\]
Moreover, by Lemma~\ref{lem:score-dilution}, if at least
$m$ distractors satisfy $z_{t,j}\!\ge\! z_{t,j^\star}-\Delta$, then $\alpha_{t,j^\star}\le 1/(1+m e^{-\Delta})$, yielding the same bound with $\varepsilon\!=\!1/(1+m e^{-\Delta})$.
\end{corollary}

Proposition~\ref{prop:needle-signal} shows the fraction of needle signal any generated token can carry is \emph{at most} its own attention mass on the needle. Under dilution (small margin), this mass is provably tiny (Corollary~\ref{cor:needle-signal-eps}), so attending to thinking tokens cannot materially increase the final answer’s effective margin unless some intermediate token first assigns non-trivial attention to the needle.

\begin{takeaway}[Takeaways:]
\textbf{(i)} With fixed weights, worst-case retrieval requires a logit margin that grows like $\Omega(\log T)$; failing to achieve this leads to score dilution and vanishing $\alpha_{i,j^\star}$.
\textbf{(ii)} Autoregressively generating additional tokens with the same static attention does not repair missing access to the evidence.
\textbf{(iii)} Any successful inference-time strategy must change the similarity $q_i^\top k_j$ (e.g., by updating queries) rather than sampling more tokens with unchanged parameters.
\end{takeaway}

\section{Efficient Test-Time Adaptation via Query-Only Updates}
\label{sec:method}

Having established that existing inference-time
scaling strategies on vanilla transformer models
fail for long contexts, we now investigate an
alternate strategy of allocating inference-time compute
via test-time training (TTT).
First, we establish why a standard TTT approach,
involving several forward and backward passes over
the model,
is computationally infeasible for long contexts.
We introduce \method{} (\shortmethod{})
that captures the benefits of TTT while
minimizing the computational overhead
by re-using the KV cache
and only changing the query projections.
We present theoretical (\S\ref{subsec:why-works})
and empirical (\S\ref{sec:results}) evidence
for the efficacy of \shortmethod{}
over vanilla ICL and thinking tokens.

\looseness=-1
\paragraph{Naïve Test-Time Training is Infeasible for Long Contexts.}
A natural first-step is full-parameter TTT:
update FFN and all attention projections ($W_Q,W_K,W_V$)
on the long input $x_{1:T}$.
We find that this is impractical for long-context regimes:
every update alters keys/values across the sequence,
invalidating the KV cache and forcing fresh forward–backward
passes over the \emph{entire} context at each step,
with prohibitive compute and activation memory.

Compute-wise, our FLOP calculations (\autoref{app:flops}) shows that even \emph{one} such full-parameter TTT step over a $T$-token context 
is equivalent to generating about $1.2\times T$ decoding tokens.
That is, for a context of about $T\approx 10^5$ tokens,
this makes a single training step FLOP equivalent
to generating $\sim120$K decoding tokens---rendering
full-parameter TTT untenable.

\looseness=-1
These constraints motivate a cache-preserving alternative. Our approach, \method{} (\shortmethod{}), performs a single prefill to cache $\{K,V\}$ and then adapts \emph{only} the query projections on short spans, keeping the attention evidence pathway fixed while reshaping access to it. This retains the benefits of TTT without repeated full-context passes; we describe and formalize this procedure next.
% and make FLOP comparisons to standard in-context decoding-based baselines in \S\ref{subsec:flops-summary}.

\subsection{Query-Only TTT for Long Context}
\begin{algorithm}[t]
\caption{Query-Only Test-Time Training for Long Context}
\label{alg:method}
\begin{algorithmic}[1]
\State \textbf{Input:} model $f_\theta$, long context $x_{1:T}$, number of steps $N_{\text{TTT}}$, span length $k$, step size $\eta$
\State $\{K^{(\ell)}, V^{(\ell)}\}_{\ell=1}^L \leftarrow$ \textsc{ForwardPassAndCache}$(f_\theta, x_{1:T})$ \Comment{Single $O(T^2)$ operation}
\For{$n=1$ to $N_{\text{TTT}}$}
  \State Sample a random span $x_s = x_{t:t+k}$ from $x_{1:T}$
  \State Compute $\mathcal{L}_{\text{TTT}}(\theta; x_s)$ using the frozen $\{K^{(\ell)}, V^{(\ell)}\}$
  \State Update only the query parameters: $\{W_Q^{(\ell)}\} \leftarrow \{W_Q^{(\ell)}\} - \eta \, \nabla_{\{W_Q^{(\ell)}\}} \mathcal{L}_{\text{TTT}}$
\EndFor
\State \textbf{return} adapted model $f_{\theta'}$ to generate the final answer
\end{algorithmic}
\end{algorithm}

\begin{wrapfigure}[18]{r}{0.45\textwidth} % [18]=approx lines to wrap; tweak if needed
% \vspace{-0.8\baselineskip} % pull figure up a bit to align with heading
\centering
\vspace{-0.5cm}
\includegraphics[width=\linewidth]{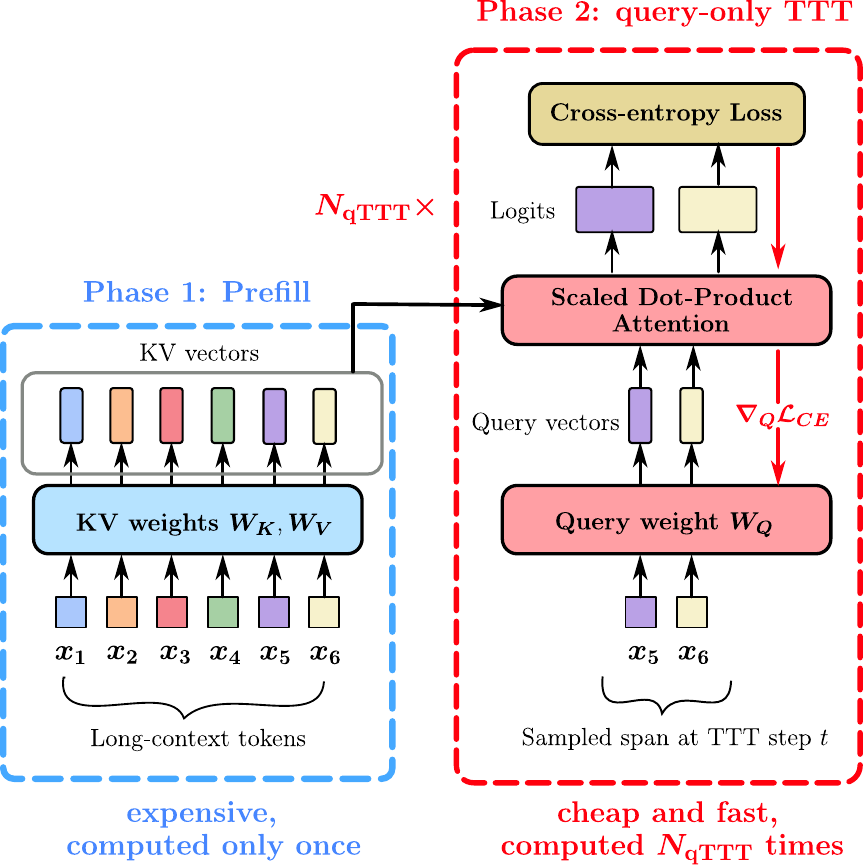}
\captionsetup{font=small}
\caption{\label{fig:qttt} Overview of \method{}.}
% \vspace{-0.6\baselineskip} % reduce gap after the figure
\end{wrapfigure}

% Please fix N_{qTTT} -> N_{TTT} in Figure 2

The core idea of \method{} % the TTT approach we cover here
is to avoid repeated, costly forward
and backward passes over the long context.
Instead, we perform a single expensive prefill to cache the context's key and value representations and then execute a series of much cheaper, targeted gradient updates.
The procedure, also outlined in
Algorithm~\ref{alg:method} and Figure~\ref{fig:qttt}, is as follows:

\begin{itemize}[leftmargin=*]
    \item[\textbf{1.}] \textbf{Single-Pass KV Cache Generation.} Given a long context $x_{1:T}$, we perform exactly one full forward pass with the pre-trained model $f_\theta$. During this pass, for each layer $\ell$ in the model, we compute and store the Key and Value projection tensors, $K^{(\ell)} \in \mathbb{R}^{T \times d_k}$ and $V^{(\ell)} \in \mathbb{R}^{T \times d_v}$. These cached tensors represent the complete contextual information and remain frozen for the duration of the adaptation process.

    \item[\textbf{2.}] \textbf{Span-Sampled, Query-Only Objective.} With the KV cache held constant, we perform $N_{\text{TTT}}$ steps of gradient descent. In each step, we update only the query projection matrices $\{W_Q^{(\ell)}\}_{\ell=1}^L$. The objective is the standard next-token prediction loss, computed over a small, randomly sampled contiguous span of tokens $x_s = x_{t:t+k}$, where the span length $k \ll T$: % The TTT loss is defined as:
    \begin{align}
    \mathcal{L}_{\text{TTT}}(\theta; x_s)
    = - \sum_{i=t}^{t+k-1} \log p_\theta(x_{i+1}\mid x_{1:i}; \{K^{(\ell)}, V^{(\ell)}\}_{\ell=1}^L)
    \end{align}
    Crucially, the gradients $\nabla_\theta \mathcal{L}_{\text{TTT}}$ are computed and applied only with respect to the parameters $\{W_Q^{(\ell)}\}$, leaving all other model weights, including the now-static KV cache, unchanged.
    % This targeted update trains the model to formulate better ``questions" (queries) to ask of the fixed context (keys and values).
\end{itemize}

\subsection{Why Query-Only Test-Time Training is Effective}
\label{subsec:why-works}

\begin{figure}[t]
    \centering
    \includegraphics[width=0.8\linewidth]{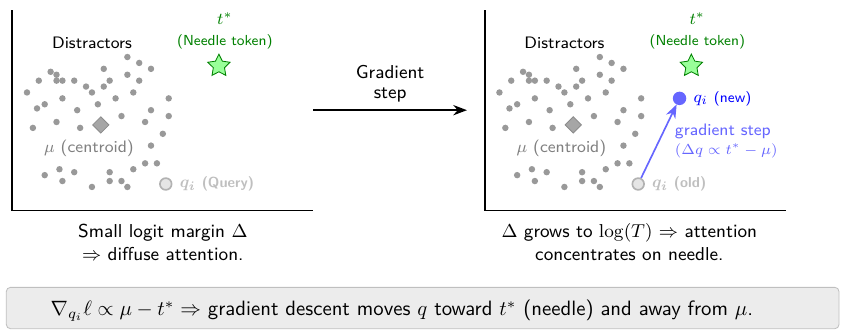}
    \caption{A visual representation of Proposition~\ref{claim:query-gradient} showing
    how \shortmethod{} improves the logit margin. The gradient updates via \shortmethod{} directly move the query projection weights towards the target needles and counteracts score dilution.
    }
    \label{fig:qttt_margin}
\end{figure}

\looseness=-1
\Cref{sec:failures} showed that long-context failures
arise from score dilution and the resulting need for a
growing target–distractor \emph{margin}.
Query-only TTT targets this bottleneck directly:
only adapt the query projections % $\{W_Q^{(\ell)}\}$
while holding keys/values fixed (from a single prefill).
This leaves the evidence (K,V)
unchanged and instead reshapes \emph{query} to it by 
modifying the similarity $q_i^\top k_j$ for a given input (Proposition~\ref{claim:query-gradient}; \Cref{fig:qttt_margin}).

\looseness=-1
\begin{proposition}[Query update]
\label{claim:query-gradient}
For loss $\ell_i=-\log\alpha_{i,j^\star}$ with fixed $K$, the gradient w.r.t.\ $q_i$ is
\[
\nabla_{q_i}\ell_i=\frac{1}{\sqrt{d_k}}\Big(\underbrace{\sum_{\ell=1}^T \alpha_{i,\ell}k_\ell}_{\mu_i}-k_{j^\star}\Big).
\]
A descent step $q_i\leftarrow q_i-\eta\nabla_{q_i}\ell_i$ moves $q_i$ toward $k_{j^\star}$ and away from the attention-weighted mean $\mu_i$, explicitly counteracting dilution. (The statement holds per head and aggregates across heads.)
\end{proposition}

\begin{lemma}[Margin improvement]
\label{lemma:margin-improvement}
Let $M_i(q_i)\coloneqq -\ell_i(q_i)$ denote the logit margin. For sufficiently small $\eta>0$,
\[
M_i\big(q_i-\eta\nabla_{q_i}\ell_i\big)
= M_i(q_i) + \eta\|\nabla_{q_i}\ell_i\|_2^2 + O(\eta^2).
\]
Hence the margin strictly increases whenever $\nabla_{q_i}\ell_i\neq 0$, with the gain proportional to $\|k_{j^\star}-\mu_i\|_2^2$. Improvements are therefore largest precisely when attention is most diffuse, i.e., in the long-context regimes where score dilution is severe.
\end{lemma}

\looseness=-1
\begin{takeaway}[Takeaway:]{Query-only TTT reallocates inference-time compute into \emph{margin-raising} updates: with fixed $\{K,V\}$ from a single prefill, each step moves $q_i$ toward $k_{j^\star}$ and \emph{provably} increases the target–distractor logit margin. It thus directly mitigates score dilution, most when attention is most diffuse, without re-encoding the context or growing the KV cache.}\end{takeaway}

\subsection{FLOP Equivalence: Thinking Tokens vs.\ Query-Only TTT}
\label{subsec:flops-summary}
We compare two ways to spend inference-time compute after a single prefill: (i) generate $T_{\text{think}}$ \emph{thinking} tokens with frozen weights, or (ii) run $N_{\text{\shortmethod{}}}$ \emph{query-only} updates on spans of length $k\!\ll\!T$ while reusing the KV cache. For long $T$, FLOP equivalence (\autoref{app:flops}) yields the rule of thumb
\begin{equation}
\label{eq:think-vs-qttt}
T_{\text{think}} \;\approx\; 2\,N_{\text{\shortmethod{}}}\,k
\qquad\text{(long $T$, span $k\!\ll\!T$).}
\end{equation}
% i.e., one span update over $k$ tokens costs about the same as decoding $\approx 2k$ tokens.

Consider a dense model of about $8$B parameters on a long context $T=10^5$
and an inference-time budget budget to decode
$8$K thinking tokens after the prefill.
From~\cref{eq:think-vs-qttt}, the FLOPs equate to about
$N_{\text{\shortmethod{}}}\!=\!16$ \method{} steps on spans of $k\!=\!128$,
% (since $2\cdot 10\cdot 400\!\approx\!8000$).
% If we prefer $k\!=\!256$,
and
% the same budget supports
$N_{\text{\shortmethod{}}}\!=\!8$
for $k\!=\!512$.
In both cases, thinking tokens grow the KV cache
by thousands of positions without changing attention,
whereas query-only TTT keeps the cache length fixed at $T$
and uses the matched FLOPs to \emph{reshape queries}
against the existing keys/values, directly
targeting the margin bottleneck from \S\ref{sec:failures}.

\section{Experimental Results}
\label{sec:results}

In this section,
we discuss experimental results
across a suit of long-context tasks.
Firstly, we callback the synthetic long-context
setup from~\S\ref{sec:empirical-limits}.
\autoref{fig:figure_1} shows
that spending inference-time compute via
\method{} results in significant performance
improvements on top of just in-context decoding.
We observe that the improvements are consistent
across context lengths unlike thinking tokens
that show rapid diminishing returns.
In the rest of this section,
we discuss our findings on 
long-context benchmarks that involve 
nuanced $n$-hop retrieval, reasoning, and comprehension.

Further, we empirically verify that
these improvements with \shortmethod{} are indeed
a result of margin improvement and reducing score dilution.
Appendix~\ref{app:score-dilution} (\Cref{tab:bank-rope-ablation})
shows an analysis of attention mass on the target tokens
with and without \shortmethod{}.
Particularly, we aggregate the attention scores
for the target tokens
(well defined for these synthetic tasks)
across model layers to study the influence of
\shortmethod{} against vanilla attention.
We observe that as number of input tokens increases
(hence the number of distractors),
the performance as well as attention mass
for vanilla attention
goes down drastically.
However, \shortmethod{} helps preserve attention mass
significantly across context lengths.

% ---------- (a) LongBench-v2: 3 models ----------
\begin{figure*}[t]
  \centering
  \begin{subfigure}[t]{0.49\textwidth}
    \centering
    \includegraphics[width=\linewidth]{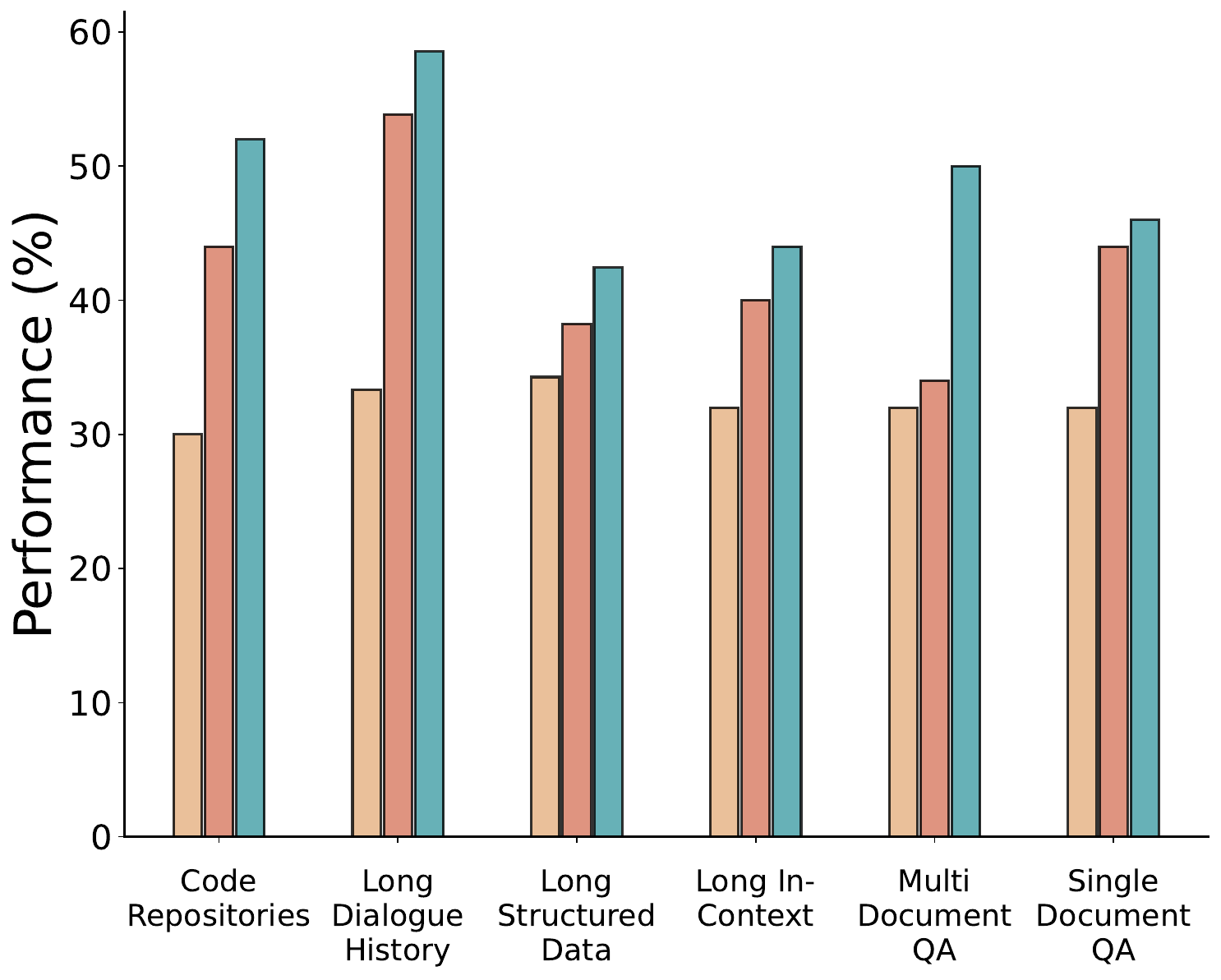}
    \small
    \mbox{\hspace*{1.5cm}\cblock{229}{177}{129} \hspace{0.5mm}In-Context Only\hspace{3mm} \cblock{215}{122}{97} \hspace{0.5mm}With Thinking\hspace{3mm} \cblock{66}{158}{166}\hspace{1mm}With Query-only Test-Time Training (qTTT)}
    \subcaption{
    \label{fig:longbench-v2-8b-bars}
    Comparison on LongBench-v2 subsets
    for Qwen3-8B.
    Using \shortmethod{} consistently
    outperforms standard in-context
    and FLOP-matched thinking settings.}
  \hfill
  \end{subfigure}
  \begin{subfigure}[t]{0.49\textwidth}
    \centering
    \includegraphics[width=\linewidth]{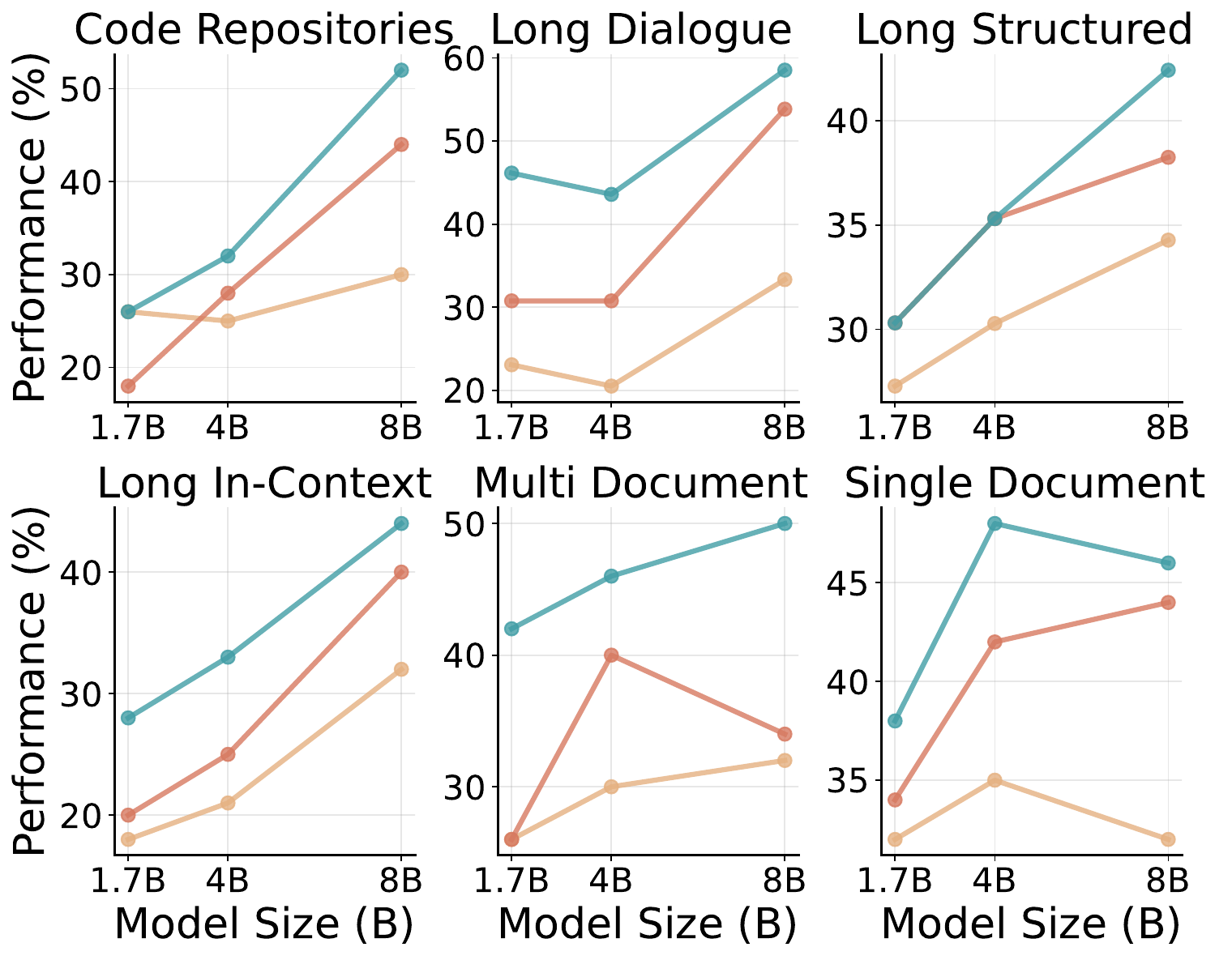}
    \vspace{.04cm}
    \subcaption{
    \label{fig:longbench-v2-size-lines}
    Variation of performance across model size
    on LongBench-v2 subsets.
    \shortmethod{} improves performance consistently
    across model sizes.}
  \end{subfigure}
  % \vspace{-0.5cm}
  \captionsetup{skip=1pt, belowskip=1pt}
  \caption{\label{fig:longbench-v2-main}
  LongBench-v2~\citep{bai2023longbench}
  provides a testbed to evaluate long-context
  abilities across a diverse set of context types.
  Here, we report evaluations across all six subsets
  of the benchmark for Qwen3-\{$1.7$/$4$/$8$B\} models.
  \shortmethod{} shows consistent improvements 
  over both standard in-context learning and
  FLOP-matched thinking tokens across
  the different context types.
}
\end{figure*}

% ---------- (a)ZeroScrolls: 3 models ----------
\begin{figure*}[t]
  \centering
  \begin{subfigure}[t]{0.49\textwidth}
    \centering
    \includegraphics[width=\linewidth]{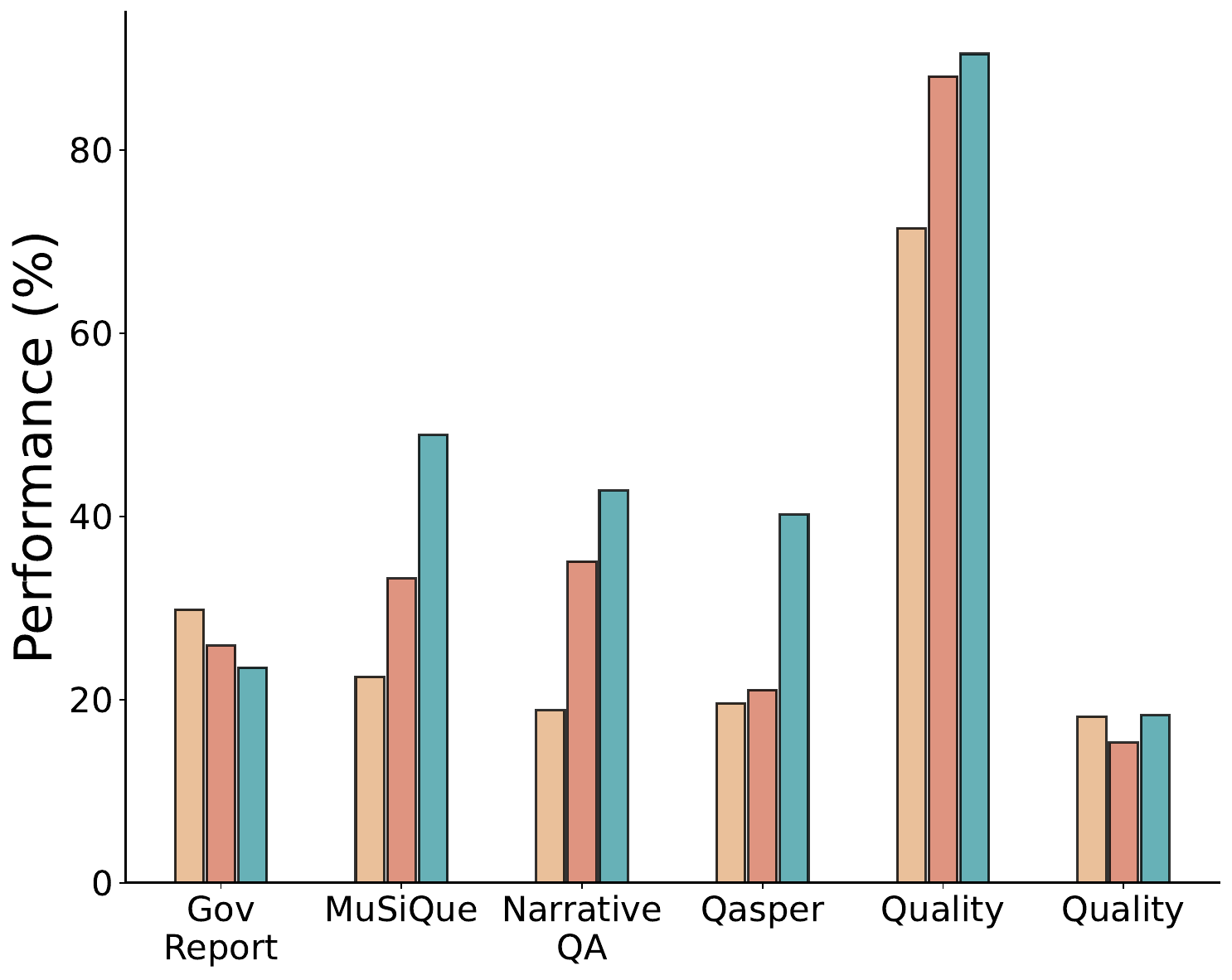}
    \small
    \mbox{\hspace*{1.5cm}\cblock{229}{177}{129} \hspace{0.5mm}In-Context Only\hspace{3mm} \cblock{215}{122}{97} \hspace{0.5mm}With Thinking\hspace{3mm} \cblock{66}{158}{166}\hspace{1mm}With Query-only Test-Time Training (qTTT)}
    \subcaption{
    \label{fig:zeroscrolls-8b-bars}
    Comparison on ZeroScrolls subsets
    for Qwen3-8B.
    Using \shortmethod{} consistently
    outperforms standard in-context
    and FLOP-matched thinking settings.}
  \end{subfigure}
  \hfill
  \begin{subfigure}[t]{0.49\textwidth}
    \centering
    \includegraphics[width=\linewidth]{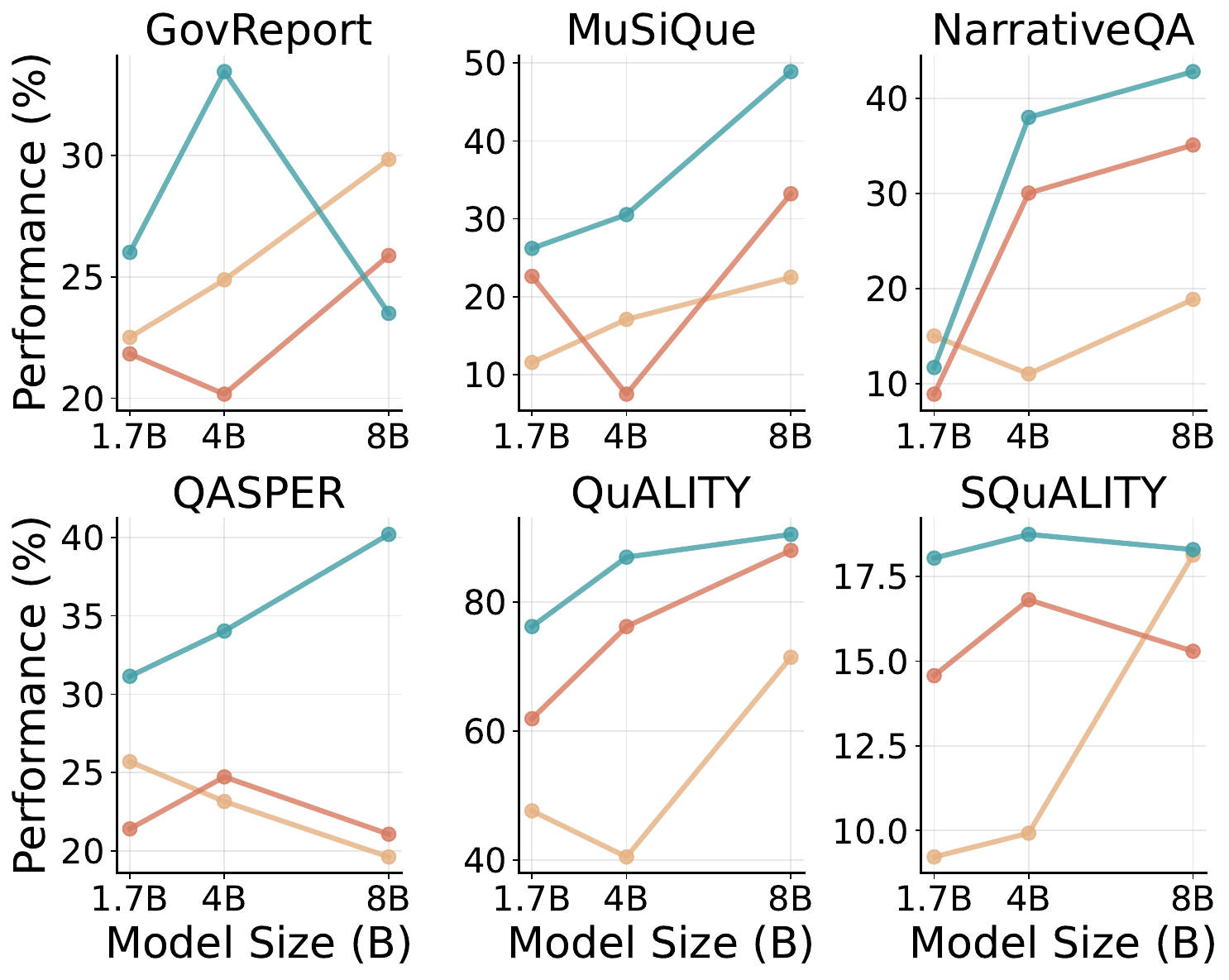}
    \vspace{.04cm}
    \subcaption{
    \label{fig:zeroscrolls-size-lines}
    Variation of performance across model size
    on ZeroScrolls subsets.
    \shortmethod{} improves performance consistently
    across sizes, often greater for larger models.}
  \end{subfigure}
  % \vspace{-0.25cm}
  % \captionsetup{skip=0pt, belowskip=0pt}
  \caption{\label{fig:zeroscrolls-main}
  ZeroScrolls~\citep{shaham2023zeroscrolls}
  evaluates a diverse set of tasks
  and model abilities over long context inputs.
  We report evaluations across six subsets
  for Qwen3-\{$1.7$/$4$/$8$B\} models.
  \shortmethod{} shows consistent improvements 
  over both standard in-context learning and
  FLOP-matched thinking tokens,
  especially for retrieval-based multi-hop reasoning
  and long form comprehension tasks.}
  \captionsetup{skip=1pt, belowskip=1pt}
\end{figure*}

\para{Setup and Evaluation Protocol.}
We evaluate \method{} (\shortmethod{}) on long-context tasks against two baselines: (i) \emph{In-context}—standard decoding with no intermediate tokens; and (ii) \emph{Thinking}—a chain-of-thought variant whose extra tokens are \emph{compute-matched} to \shortmethod{} via the FLOP equivalence in \S\ref{subsec:flops-summary}.
Our experiments are performed over
Qwen3 models across 1.7B, 4B, and 8B parameters,
and cover all subsets of \textbf{LongBench-v2}~\citep{bai2023longbench} (six categories) and \textbf{ZeroSrolls}~\citep{shaham2023zeroscrolls} (eight datasets).
% Unless otherwise noted, prompts and decoding settings are identical across methods; the \emph{Thinking} baseline differs only in allocating its matched budget to intermediate tokens, while \shortmethod{} reuses a single long-context KV cache and applies query-only span updates under the same FLOP budget. Detailed hyperparameters are provided in \autoref{app:flops}.
% 
% To ensure fairness under a fixed inference-time budget, we use the FLOP equivalence from \S\ref{subsec:flops-summary}.
Unless stated otherwise, we use $T_{\text{think}}{=}8192$, $k{=}128$, $N_{\text{\shortmethod{}}}{=}32$,
and a common budget of $512$ tokens to generate
the final answer\footnote{We use the /think and /no\_think tokens in the Qwen3 model to control for this. We elaborate on further details
including decoding parameters and
prompt templates in \autoref{app:experimental-details}.}.

\para{LongBench-v2.}
LongBench-v2~\citep{bai2023longbench}
evaluates long-context reasoning across 
diverse context types.
The benchmark probes 
whether models can locate and use dispersed evidence
to answer multiple-choice questions
across a variety of context types:
given multi-file project trees in the
\textit{Code Repositories} setting,
to resolve arguments of a particular function; and
% (e.g., ``What are the arguments for the \texttt{foobar\(\)} function?").
% Whereas, for ,
given the context as a set of related documents
in the \emph{Multi-Document QA} setting,
% drawn from multiple sources;
% and the task is to
synthesize spans across sources to answer a question.
% (``Based on the background memo and follow-up report, which committee issues the final verdict?").
This allows us to assess the applicability of \shortmethod{}
across forms of input contexts.

\autoref{fig:longbench-v2-main} shows that, under compute-matched budgets, \shortmethod{} delivers consistent and often substantial gains across model sizes. On \emph{Long Dialogue History} and \emph{Multi-Document QA}, where evidence is most diffuse, \shortmethod{} outperforms standard in-context and thinking by wide margins (e.g., for Qwen3-4B: 30.8 $\rightarrow$ \textbf{43.6} on \emph{Long Dialogue History}; 40.0 $\rightarrow$ \textbf{46.0} on \emph{Multi-Document QA}). In \emph{Code Repositories}, \shortmethod{} scales especially well with model size (for Qwen3-8B: 30.0 $\rightarrow$ 44.0 $\rightarrow$ \textbf{52.0}).
% suggesting improved retrieval and preservation of relevant spans as capacity grows.
Overall, the LongBench-v2 results indicate that \shortmethod{} fares well across markedly different context types.
% corroborating the desirable properties discussed in \S\ref{subsec:why-works}.

\para{ZeroScrolls.}
ZeroScrolls \citep{shaham2023zeroscrolls} evaluates long-context reasoning across diverse tasks. We group the datasets into three categories: (i) \emph{Multi-hop reasoning} (\texttt{MuSiQue}, \texttt{QASPER}, \texttt{NarrativeQA}), which require locating and composing evidence spread across long documents;
% and are scored with the task’s official QA metrics (EM/F1 or equivalent)
(ii) \emph{Long-form summarization} (\texttt{GovReport}, \texttt{QMSum}, \texttt{SQuALITY}), which emphasize distilling lengthy inputs;
and (iii) \emph{Long-passage comprehension} (\texttt{QAuLITY}), which measures multiple-choice accuracy over extended contexts.
% This suite of tests allows us to evaluate whether models can not only ingest long inputs but also retrieve,
% integrate, and use the right spans during generation.
In contrast to LongBench-v2, this suite of tests
evaluates the ability to utilize some long context
to solve a variety of different tasks.

\autoref{fig:zeroscrolls-overview} shows that
\shortmethod{} consistently outperforms
% both \emph{In-context} and \emph{Thinking}
vanilla thinking
on multi-hop QA and comprehension tasks,
with gains that strengthen with model size.
% This aligns with our hypothesis developed in~\S\ref{subsec:why-works}.
% adapting queries increases the target–distractor margin, directly counteracting score dilution, whereas generating more tokens leaves the attention kernel unchanged.
On summarization-style datasets,
improvements are smaller
and comparable to thinking, suggesting that
when generation quality, not retrieval, is the
primary bottleneck, reweighting attention yields limited returns. 
Overall,
% results on ZeroScrolls show \shortmethod{} yields
we see significant performance gains
across datasets and model scales.
% while using the same FLOP budget.
% This strengthens the claim that reallocating
% compute to query-only adaptation is a reliably
% better use of resources in long-context regimes.

The full set of results on LongBench-v2 and ZeroScrolls
are elaborated in Appendix~\ref{app:all_results}.
Moreover, we include additional test-time compute baselines
such as best-of-N and beam search
in Appendix~\ref{app:tts-baselines}.
We also perform a comprehensive latency and wall-clock
time comparison of \shortmethod{} with other approaches
in Appendix~\ref{app:latency}.

\begin{takeaway}[Takeaways:]{\textbf{(i)} We see consistent gains in performance across benchmarks and model sizes, \shortmethod{} yields the best average performance under matched FLOPs (\autoref{fig:longbenchv2-overview}, \autoref{fig:zeroscrolls-overview}).
\textbf{(ii) }Retrieval-driven tasks benefit the most, validating the score dilution diagnosis and the margin increase with \shortmethod{} (\S\ref{sec:failures}, \S\ref{subsec:why-works}). \textbf{(iii) }Thinking tokens are not a reliable substitute: they sometimes help but can also trail \emph{In-context}, especially in very long contexts. \textbf{(iv)} \shortmethod{} is a more effective use of inference-time compute; without changing architecture, data, or pre-training.}
\end{takeaway}

% \vspace{2mm}
% \noindent\textit{Reproducibility checklist (to be added).} Model checkpoints and tokenizer versions; context window and truncation policy; prompts; decoding hyperparameters; $T_{\text{think}}$, $N_{\text{TTT}}$, $k$; optimizer and LR for \shortmethod{}; precision (bf16/fp16), gradient clipping; hardware; evaluation scripts and commit hashes.
\section{Prior Work}
\label{sec:prior}

\looseness=-1
\para{Long-Context LLMs.}
Context windows have expanded rapidly, with models reaching million-token scale~\citep{reid2024gemini},
usually extending limits via RoPE scaling~\citep{chen2023extending,bai2023qwen}.
Parallel efforts reduce quadratic attention with sparse/structured patterns \citep{beltagy2020longformer,zaheer2020bigbird}.
Evaluation has coalesced around long-context suites such as LongBench/LongBench-v2 \citep{bai2023longbench}, ZeroScrolls~\citep{shaham2023zeroscrolls}, RULER, and domain-specific code benchmarks like SWE-bench variants \citep{jimenez2024swebench}.
% \para{Failure modes in long contexts.}
However, these LLMs still exhibit
strong position sensitivity,
yielding the ``lost in the middle" effect \citep{liu2023lost}.
Needle-in-a-haystack–style tests show
that a single relevant span can be overwhelmed
by many distractors,
and this persists across languages and document structures \citep{kamradt2024needle}.
Our work targets this retrieval
failure by addressing how attention mass is allocated over very long inputs.

\para{Inference-Time Compute Scaling.}
A common approach is to spend more compute at inference via chain-of-thought~\citep{wei_chain--thought_2023}, self-consistency \citep{wang_self-consistency_2023}, best-of-$n$ \citep{nakano2021webgpt},
or other strategies~\citep{zelikman2024quiet, zweiger_self-adapting_2025, kang_scalable_2025}. While often helpful, these methods scale decoding and can be compute-heavy with diminishing returns \citep{snell_scaling_2024,liu_can_2025}.
% We instead allocate the same budget to \emph{adapt}
% the model’s access to evidence rather than generate more tokens.
% 
% \para{Test-time training (TTT).}
Another way to spend inference-compute is via
test-time training~\citep{sun_test-time_2020, hardt_test-time_2024,akyurek_surprising_2025}.
While typically done to handle distribution shifts,
recent work has started focusing on long-context
LLM use cases~\citep{sun_learning_2025, zuo_ttrl_2025}.
% adapts models at inference using gradient updates, typically to handle distribution shift
To our knowledge, our work is first to re-purpose
TTT to micro-distribution of individual inputs via
% introducing
a query-only variant tailored to long-context.
% To our knowledge, this is the first systematic treatment of efficient TTT for long-context comprehension in LLMs.
\section{Discussion}
\label{sec:conclusion}
We identify score dilution in static quadratic attention as a core cause of long-context failures.
We design synthetic tasks to study long-context behavior controllably and show that accuracy falls sharply with context length $T$ and “thinking’’ tokens show diminishing returns (\S\ref{sec:failures}).
We proposed \method{} (\shortmethod{}) to reallocate inference-time budget via few query-only updates that provably increase the target–distractor margin (\S\ref{sec:method}).
Under matched FLOPs, \shortmethod{} consistently outperforms \emph{in-context} and \emph{thinking} on LongBench-v2 and ZeroSCROLLS,
with the largest gains on retrieval and multi-hop reasoning (\S\ref{sec:results}). In short, adapting queries is a more effective use of inference-time compute than generating more tokens for long context tasks.

\para{Future directions.}
(1) We evaluate a single point on the $(k, N_{\text{TTT}})$ trade-off; exploring budget schedules across span size and steps is immediate. 
(2) Our compute-matched baseline focuses on “thinking’’ tokens; extending to self-consistency and best-of-$n$ within the same framework is future work. 
(3) Gains are task-dependent; developing simple predictors for when to prefer \shortmethod{} (vs.\ decoding-based scaling) is a practical next step.

\section{Acknowledgments}

This work was done when RB, RT, SSD, and DK
were summer interns at Meta.
RB would like to thank other interns
in the legacy GenAI team for
the exchange of ideas and brainstorming
that shaped this project.
Namely: Irene Zhang, Winnie Yang, Julian Coda-Forno,
Sriyash Poddar, Arushi Rai,
and others in the Research Club.
We thank Sharan Narang, 
Prateek Yadav, and Mike Lewis
for their guidance.
RB would like to thank Yonatan Belinkov,
Nihal Nayak, Lyndon Lam, Sunny Qin, Bingbin Liu, and other members
of the ML Foundations group and
the Kempner Institute at Harvard for their feedback
on the manuscript.

\clearpage
\newpage
\bibliographystyle{assets/plainnat}
\bibliography{paper}

\clearpage
\newpage
\beginappendix
\section{Synthetic Tasks}
\label{sec:dataset-examples}

\begin{figure}[!t]
   \begin{bluebox}
    \raggedright
    
        \textbf{Bug Description:} The attention mechanism fails to properly normalize attention scores, leading to numerical instability and gradient explosion during training. The attention weights grow unbounded, causing immediate training divergence.
        
        \vspace{0.7em}
        \textbf{Code context:}%\\[-4pt]
        \begin{alltt}\ttfamily\footnotesize
olmo/model.py
L335: def _scaled_dot_product_attention(
L336:\hspace{1cm}self,
L337:\hspace{1cm}q: torch.Tensor,
L338:\hspace{1cm}k: torch.Tensor,
L339:\hspace{1cm}v: torch.Tensor,
L340:\hspace{1cm}attn_mask: Optional[torch.Tensor] = None,
L341:\hspace{1cm}dropout_p: float = 0.0,
L342:\hspace{1cm}is_causal: bool = False,
L343:\hspace{1cm}) -> torch.Tensor:
L344:
L345: \hspace{.8cm}attn_weights = torch.matmul(q, k.transpose(-2, -1))
L346:
L347: \hspace{.8cm}if is_causal:
L348: \hspace{1.4cm}assert attn_mask is None
L349: \hspace{1.4cm}query_len, key_len = q.shape[-2], k.shape[-2] 
        \end{alltt}
        
        \vspace{0.3em}
        
        \textbf{Target:} Given the above code context, please identify the exact location of the bug.
        
        % \\ Output your answer in the following JSON format:
        
        % % In body
        % \begin{lstlisting}[language=json, frame=none, framerule=0pt]
        % {
        %   "bug_location": "L{line_number}",
        % }
        % \end{lstlisting}
        
        \vspace{0.0em}
        {\color{blue!60}\hdashrule{\linewidth}{0.7pt}{3mm 1.2mm}}
        \vspace{-0.2cm}
        
        \textbf{Model output:} $\texttt{olmo/model.py:L345}$
        \end{bluebox}
    \caption{\label{fig:code_example}
    An example of the code bug localization synthetic task.}
\end{figure}

\begin{figure}[!t]
    \begin{redbox}
    \raggedright
    
        \textbf{Task Description:} Analyze this banking transaction log for bugs.

        \vspace{0.7em}
        
        \textbf{Initial state:} {\small $\texttt{\{"account\_A": 4000, "account\_B": 4200, "total": 8200\}}$}

        \vspace{0.7em}

        \textbf{Rules:} {1. Total money must remain constant (conservation)}

        \hspace{1.05cm}{2.\; No account can go negative}

        \hspace{1.05cm}{3.\; All calculations must be mathematically correct}
        
        \vspace{0.7em}
        \textbf{Transaction logs:}%\\[-4pt]
        \begin{alltt}\ttfamily\footnotesize
[TX001]: Transfer \$107: A=4000 → 3893, B=4200 → 4307
[TX002]: Transfer \$204: A=3893 → 3689, B=4307 → 4511 
[TX003]: Transfer \$780: A=3689 → 2909, B=4511 → 5291
[TX004]: Transfer \$2925:A=2909 → -16,  B=5291 → 8216 
\mbox{[TX005]: Transfer \$699: B=8216 → 7517, A=-16  →  683} 
        \end{alltt}
        
        \vspace{-0.3cm}

        \textbf{Possible bug types} (choose exactly one):\\
            \hspace{.2cm}-- \texttt{CALC\_ERROR}: Mathematical calculation is incorrect\\
            \hspace{.2cm}-- \texttt{NEGATIVE\_BAL}: Account balance becomes negative\\
            \hspace{.2cm}-- \texttt{LOST\_UPDATE}:  Concurrent update causes lost transaction\\
            \hspace{.2cm}-- \texttt{DUPLICATE\_TXN}:  Same transaction processed multiple times

        \vspace{0.2cm}
        
        \textbf{Target}: Please identify the bug type and location.
%         Identify the SINGLE bug if any exists.
% Respond with ONLY a JSON object:
% {"bug_type": "<type>", "bug_location": "<TX_ID>"
        % \\ Output your answer in the following JSON format:
        
        % % In body
        % \begin{lstlisting}[language=json, frame=none, framerule=0pt]
        % {
        %   "bug_location": "L{line_number}",
        % }
        % \end{lstlisting}
        
        \vspace{0.0em}
        {\color{red!60}\hdashrule{\linewidth}{0.7pt}{3mm 1.2mm}}
        \vspace{-0.2cm}
        
        \textbf{Model output:} {\small $\texttt{\{"bug\_type": NEGATIVE\_BAL, "bug\_location": TX004\}}$} 
        \end{redbox}
    \caption{\label{fig:logs_example}
    An example of the log transactions synthetic task.
    }
\end{figure}

We illustrate two representative synthetic tasks used in our study. 
Figure~\ref{fig:code_example} shows the \emph{code bug localization} task: the model receives a brief natural-language bug description together with a minimal, line-numbered code context and must return the exact file-and-line of the offending statement. In the example, the model correctly identifies the line where attention scores are computed without proper normalization (\texttt{olmo/model.py:L345}).

Figure~\ref{fig:logs_example} shows the \emph{transaction-log consistency} task: given an initial account state, a set of invariants (e.g., conservation of total funds, no negative balances), and a short sequence of transfers, the model must select a single bug type and pinpoint the first offending transaction. In the example, the model outputs \texttt{NEGATIVE\_BAL} at \texttt{TX004}, where the balance of account A becomes negative, violating the stated rules.

Together, these examples illustrate the input/output format of our synthetic tasks, the kind of structured context provided to the model, and the expected concise targets (a specific line for code or a \{\texttt{bug\_type}, \texttt{TX\_id}\} pair for logs). We use similarly formatted instances throughout our evaluation.

\section{Proofs for Section~\ref{sec:failures}}
\label{app:proofs-long-context}

\para{Notation.}
For a fixed query $q_i$, logits are $z_{i,j}=\frac{q_i^\top k_j}{\sqrt{d_k}}$, attention weights $\alpha_{i,j}=\frac{e^{z_{i,j}}}{\sum_{\ell}e^{z_{i,\ell}}}$, and $o_i=\sum_j \alpha_{i,j}v_j$. We write $\mu_i=\sum_{\ell}\alpha_{i,\ell}k_\ell$.

\begin{proof}[Proof of Lemma~\ref{lem:log-margin} (Score dilution)]
Let $S=\{j\neq j^\star:\, z_{i,j}\ge z_{i,j^\star}-\Delta\}$ with $|S|=m$. Then
\[
\sum_{\ell=1}^T e^{z_{i,\ell}}
\;\ge\; e^{z_{i,j^\star}}+\sum_{j\in S}e^{z_{i,j}}
\;\ge\; e^{z_{i,j^\star}}\big(1+m e^{-\Delta}\big),
\]
hence $\alpha_{i,j^\star}=\frac{e^{z_{i,j^\star}}}{\sum_\ell e^{z_{i,\ell}}}\le\frac{1}{1+m e^{-\Delta}}$. 
If $m\ge cT$ with $c>0$ and $\Delta=O(1)$, then $\alpha_{i,j^\star}\to 0$ as $T\to\infty$.
\end{proof}

\begin{proof}[Proof of Lemma~\ref{lem:log-margin} (Logarithmic margin requirement)]
Let $\gamma=\min_{j\ne j^\star}(z_{i,j^\star}-z_{i,j})$. Then 
$\sum_{j\ne j^\star}e^{z_{i,j}}\le (T-1)e^{z_{i,j^\star}-\gamma}$, so
\[
\alpha_{i,j^\star}
=\frac{1}{1+\sum_{j\ne j^\star}e^{z_{i,j}-z_{i,j^\star}}}
\;\ge\; \frac{1}{1+(T-1)e^{-\gamma}}.
\]
Rearranging $\frac{1}{1+(T-1)e^{-\gamma}}\ge 1-\varepsilon$ yields 
$\gamma \ge \log\!\big(\tfrac{(T-1)(1-\varepsilon)}{\varepsilon}\big)$.
\end{proof}

\begin{proof}[Proof of Proposition~\ref{prop:needle-signal} (Needle-signal bound)]
For any thinking token $t$,
\[
o_t=\sum_{j<t}\alpha_{t,j}v_j
= \alpha_{t,j^\star} v_{j^\star} + (1-\alpha_{t,j^\star}) \sum_{j\neq j^\star}\tilde\alpha_{t,j} v_j,
\quad
\tilde\alpha_{t,j}=\frac{\alpha_{t,j}}{1-\alpha_{t,j^\star}}.
\]
For any $u\in\mathbb{R}^{d_v}$, take inner products and upper bound the convex combination by its maximum term:
\[
\big\langle u,o_t\big\rangle
\le
\alpha_{t,j^\star}\,\langle u,v_{j^\star}\rangle
+
(1-\alpha_{t,j^\star})\,\max_{j\ne j^\star}\langle u,v_j\rangle.
\]
\end{proof}

\begin{proof}[Proof of Corollary~\ref{cor:needle-signal-eps} (Specialization under small margin)]
By Definition~\ref{def:retrieval}, $\gamma_t \le \log\!\big(\varepsilon/(1-\varepsilon)\big)$ iff $\alpha_{t,j^\star}\le \varepsilon$. Substitute $\alpha_{t,j^\star}\le\varepsilon$ in Proposition~\ref{prop:needle-signal} to obtain
\[
\langle u,o_t\rangle \le \varepsilon\langle u,v_{j^\star}\rangle + (1-\varepsilon)\max_{j\ne j^\star}\langle u,v_j\rangle.
\]
Moreover, Claim~\ref{lem:log-margin} implies $\alpha_{t,j^\star}\le 1/(1+ m e^{-\Delta})$ when at least $m$ distractors satisfy $z_{t,j}\ge z_{t,j^\star}-\Delta$, yielding the bound with $\varepsilon = 1/(1+m e^{-\Delta})$.
\end{proof}

\begin{proof}[Proof of Claim~\ref{claim:query-gradient} (Directional query update)]
With $z_{i,\ell}=\frac{q_i^\top k_\ell}{\sqrt{d_k}}$,
\[
\ell_i(q_i)=-\log \alpha_{i,j^\star}=-z_{i,j^\star}+\log\!\sum_{\ell=1}^T e^{z_{i,\ell}}.
\]
Differentiating w.r.t.\ $q_i$ and using $\frac{\partial z_{i,\ell}}{\partial q_i}=\frac{k_\ell}{\sqrt{d_k}}$,
\[
\nabla_{q_i}\ell_i
= -\frac{k_{j^\star}}{\sqrt{d_k}} + \frac{1}{\sum_{\ell'}e^{z_{i,\ell'}}}\sum_{\ell=1}^T e^{z_{i,\ell}}\frac{k_\ell}{\sqrt{d_k}}
= \frac{1}{\sqrt{d_k}}\Big(\sum_{\ell=1}^T \alpha_{i,\ell}k_\ell - k_{j^\star}\Big)
= \frac{1}{\sqrt{d_k}}(\mu_i - k_{j^\star}).
\]
Thus a descent step moves $q_i$ toward $k_{j^\star}$ and away from $\mu_i$.
\end{proof}

\begin{proof}[Proof of Lemma~\ref{lemma:margin-improvement} (Monotone margin improvement)]
Define $M_i(q_i)=-\ell_i(q_i)$. Then $\nabla M_i(q_i)=-\nabla \ell_i(q_i)$. 
For a step $q_i^+=q_i-\eta\nabla \ell_i(q_i)$, a first-order expansion gives
\[
M_i(q_i^+) = M_i(q_i) + \eta\|\nabla \ell_i(q_i)\|_2^2 + O(\eta^2).
\]
Using Claim~\ref{claim:query-gradient}, 
$\|\nabla_{q_i}\ell_i\|_2^2=\frac{1}{d_k}\|k_{j^\star}-\mu_i\|_2^2$, which is strictly positive unless $k_{j^\star}=\mu_i$. 
If $\nabla \ell_i$ is $L$-Lipschitz, choosing $\eta\in(0,1/L]$ ensures $M_i(q_i^+)\ge M_i(q_i)+\tfrac{\eta}{2}\|\nabla \ell_i(q_i)\|_2^2$.
\end{proof}

\para{Remarks on multi-head attention.}
All statements apply per head. Let superscript $h$ index heads and define per-head logits/weights $\{z^{(h)}_{i,j},\alpha^{(h)}_{i,j}\}$. Claims on dilution and margin hold headwise; aggregation across heads is via concatenation and an output projection, which preserves the directional and margin-improvement arguments by linearity.
\section{FLOP Derivations for \S\ref{subsec:flops-summary}}
\label{app:flops}
We outline FLOP models for two inference-time modes and derive the equivalence summarized in Eq.~\eqref{eq:think-vs-qttt}. Consider a dense Transformer with $L$ layers, hidden size $d$, MLP ratio $r$ (so $d_{\text{ff}}=r d$), and long context length $T$. Let $T_{\text{think}}$ be the number of autoregressively generated ``thinking'' tokens, $N_{\text{\shortmethod{}}}$ the number of query-only updates, and $k$ the span size per update.

\para{Cost coefficients.}
Ignoring lower-order terms (layer norms, biases), we collect the dominant costs as
\[
C_{\text{quad}} \;=\; 2L d \quad\text{(quadratic attention term)}, 
\qquad
C_{\text{tok}} \;=\; (4{+}2r)L d^2 \quad\text{(per-token projections/MLP)}.
\]
A parallel forward over $T$ tokens (the prefill) costs
\[
F_{\text{prefill}}(T) \;=\; C_{\text{quad}}\,T^2 \;+\; C_{\text{tok}}\,T.
\]

\para{Case A (autoregressive ``thinking'').}
After one prefill, generating $T_{\text{think}}$ tokens with a KV cache costs
\[
F_{\text{gen}}(T_{\text{think}};T)
\;=\;
C_{\text{quad}}\!\left(T_{\text{think}}\,T+\frac{T_{\text{think}}(T_{\text{think}}-1)}{2}\right)
\;+\;
C_{\text{tok}}\,T_{\text{think}},
\]
so the total is $F_A \;=\; F_{\text{prefill}}(T)+F_{\text{gen}}(T_{\text{think}};T)$.

\para{Case C (\method{}: query-only with cached K/V).}
With one prefill, each query-only pass recomputes queries for $k$ positions that attend to cached $\{K,V\}$ and backpropagates only into $\{W_Q\}$. The per-pass cost is
\[
G_{\text{partial}}(k;T)
\;\approx\;
2\Big(C_{\text{quad}}\,kT \;+\; (2{+}2r)L\,k\,d^2\Big),
\]
and the total is $F_C \;=\; F_{\text{prefill}}(T) + N_{\text{\shortmethod{}}}\,G_{\text{partial}}(k;T)$.
(If the span also attends within itself, add $+\,C_{\text{quad}}k^2$ and $+\,2Lk d^2$ inside $G_{\text{partial}}$, which are dominated by $kT$ when $k\!\ll\!T$.)

\para{Equivalence (A vs.\ C).}
Cancelling the shared prefill and equating $F_{\text{gen}}(T_{\text{think}};T)=N_{\text{\shortmethod{}}}\,G_{\text{partial}}(k;T)$ yields
\[
C_{\text{quad}}\!\left(T_{\text{think}}\,T+\tfrac{T_{\text{think}}(T_{\text{think}}-1)}{2}\right)+C_{\text{tok}}\,T_{\text{think}}
\;=\;
2N_{\text{\shortmethod{}}}\,k\Big(C_{\text{quad}}\,T+(2{+}2r)L d^2\Big).
\]
For long contexts with $T\!\gg\!d$ and spans $k\!\ll\!T$ (hence $T_{\text{think}}\!\ll\!T$ in matched regimes), the dominant terms give
\[
T_{\text{think}}
\;\approx\;
2\,N_{\text{\shortmethod{}}}\,k,
\]
which is Eq.~\eqref{eq:think-vs-qttt}. First-order corrections are $O\!\big(\tfrac{T_{\text{think}}}{T}\big)$ from the $\tfrac{T_{\text{think}}(T_{\text{think}}-1)}{2}$ term and $O\!\big(\tfrac{d}{T}\big)$ from $C_{\text{tok}}$.

\para{Sanity check (numeric instantiation).}
Take $L{=}32$, $d{=}4096$, $r{=}4$ (a $\sim$7B dense model) and $T{=}10^5$. If the application budget allows decoding $T_{\text{think}}{=}8{,}000$ thinking tokens after prefill, the matched query-only schedules include, e.g., $(N_{\text{\shortmethod{}}}{=}10,\,k{=}400)$ since $2\cdot 10\cdot 400\approx 8{,}000$. This reallocation keeps the KV cache length fixed at $T$ and spends the same FLOPs to reshape queries against the existing $\{K,V\}$ instead of growing the cache with additional tokens.
\section{Experimental Details}
\label{app:experimental-details}

\para{Models and tokenization.}
We evaluate Qwen3-\{1.7B, 4B, 8B\} with their native tokenizers and maximum supported context windows. All prompts use UTF-8, and inputs are delimited with explicit section headers (e.g., \texttt{[CONTEXT]}, \texttt{[QUESTION]}). Unless otherwise noted, we evaluate on the official validation/dev splits and follow each benchmark’s scoring script.

\para{Decoding and “Thinking’’ budget.}
We adopt model-recommended decoding parameters:
% \footnote{Fill in if you deviated per-subset.}
\emph{Thinking}: temperature=0.6, top-$p$=0.95, top-$k$=20; 
\emph{Non-thinking}: temperature=0.7, top-$p$=0.8, top-$k$=20.
We cap total generation length so that \emph{Thinking} consumes exactly $T_{\text{think}}$ intermediate tokens plus the final answer; for compute matching, we use $T_{\text{think}}=8192$ unless otherwise stated.
Self-consistency/best-of-$n$ are \emph{disabled} by default to keep FLOPs matched.
% \footnote{If any subset uses voting, list $n$ and the scorer.}

\para{Query-only TTT (\method{}) hyperparameters.}
We update only $W_Q$ in all attention layers using AdamW (weight decay 0.01) with a sweep over learning rates $\{3\mathrm{e}{-4}, 3\mathrm{e}{-5}, 1\mathrm{e}{-5}, 3\mathrm{e}{-6}, 1\mathrm{e}{-6}, 3\mathrm{e}{-7}\}$; we report the best per-dataset LR selected on a held-out portion of the validation set. Batch size is 1 (long contexts). We perform $N_{\text{TTT}}$ span updates of length $k$ with a single prefill/cached $\{K,V\}$; unless stated otherwise, $(k,N_{\text{TTT}})=(128,32)$, compute-matched to \emph{Thinking} via $T_{\text{think}}\approx 2N_{\text{TTT}}k$ (\S\ref{subsec:flops-summary}). Spans are sampled uniformly over $[1, T{-}k]$; gradient clipping at 1.0; bf16 precision.
Additionally, we perform a sensitivity analysis of \shortmethod{}
across learning rates.
Table~\ref{tab:lr_sensitivity} shows the variation of
accuracy on our synthetic tasks across context lengths.
We find that qTTT is not very sensitive to the choice of LR: the performance is relatively consistent between $[1\mathrm{e}{-5}, 1\mathrm{e}{-6}]$ and only falls on the extreme values of LR.
\color{black} %

\begin{table}[ht]
\centering
% \small
% \setlength{\tabcolsep}{5pt}
\caption{\label{tab:lr_sensitivity}
\textbf{Sensitivity to Learning Rate ($\eta$).} Performance of qTTT across varying learning rates. Extreme rates cause instability (high $\eta$) or insufficient adaptation (low $\eta$), with the optimal range typically between $1\text{e-}6$ and $1\text{e-}5$.}
\begin{tabular}{l|c|c|c|c|c|c}
\toprule
\textbf{Task / Context} & \textbf{1e-4} & \textbf{3e-5} & \textbf{1e-5} & \textbf{3e-6} & \textbf{1e-6} & \textbf{3e-7} \\
\midrule
\multicolumn{7}{l}{\textit{Bank Transactions}} \\
512 & 4.2 & 26.5 & \textbf{28.0} & 27.2 & 26.8 & 15.5 \\
2,536 & 1.5 & 13.8 & \textbf{14.4} & 14.0 & 12.5 & 6.2 \\
5,120 & 0.8 & \textbf{10.0} & 9.2 & 8.5 & 7.8 & 3.5 \\
9,560 & 0.0 & 7.8 & \textbf{8.4} & 7.9 & 7.0 & 1.2 \\
\midrule
\multicolumn{7}{l}{\textit{OLMo Code Bugs}} \\
512 & 8.5 & 42.0 & 44.5 & \textbf{45.7} & 43.2 & 22.0 \\
2,050 & 5.1 & 38.5 & 40.2 & \textbf{41.6} & 39.5 & 18.5 \\
7,450 & 2.2 & 25.0 & \textbf{28.0} & 27.5 & 24.8 & 10.5 \\
10,000 & 1.0 & 18.2 & \textbf{20.2} & 19.5 & 17.8 & 5.2 \\
\bottomrule
\end{tabular}
\end{table}

% \para{Preprocessing and truncation.}
% For inputs longer than the model limit, we apply \emph{head+tail} truncation with sentinel markers:
% we retain the first $H$ tokens and last $T{-}H$ tokens, inserting \texttt{[...skipped...]} in between; $H$ is chosen so the total fits the window. For multi-document inputs, we keep document boundaries with \texttt{[DOC i]} tags. We normalize whitespace and Unicode punctuation; answers are lowercased for EM-style metrics where appropriate.

\para{Evaluation metrics.}
We use official scripts per subset: EM/F1 or dataset-specific accuracy for QA; ROUGE-\{1,2,L\} or benchmark-provided summary metrics for summarization; multiple-choice accuracy for \texttt{QAuLITY}. When a subset defines both EM and F1, we report the primary metric specified by the benchmark.

\para{Prompts and templates.}
Below we provide the base non-thinking and thinking templates used per task family. All runs share the same template within a family across methods; \emph{Thinking} adds a scratchpad section but the final answer must appear after a \texttt{Final:} tag.

\textit{Non-thinking (base)}
\begin{verbatim}
[SYSTEM]
You are a careful assistant. Use only the provided context.
If the answer is not supported, output "unknown".
[TASK]
{TASK_DESCRIPTION}    # e.g., short answer QA / summary / MCQ
[CONTEXT]
{CONTEXT_BLOCKS}     # e.g., {DOCUMENTS}|{DIALOGUE}|{CODE}|{TABLE}
[QUESTION or INSTRUCTION]
{QUESTION_OR_INSTRUCTION}     # prompt for the required output
[CONSTRAINTS]
[ANSWER]
\end{verbatim}

\textit{Thinking (base)}
\begin{verbatim}
[SYSTEM]
Reason privately in [SCRATCHPAD],
then provide a single final output after "Final:".
If not supported by the context, output "Final: unknown".
[TASK]
{TASK_DESCRIPTION}
[CONTEXT]
{CONTEXT_BLOCKS}
[QUESTION or INSTRUCTION]
{QUESTION_OR_INSTRUCTION}
[SCRATCHPAD]
...    # hidden chain-of-thought tokens (capped to T_think)
[FINAL]
Final:
\end{verbatim}

\para{Post-processing and extraction.}
For “thinking’’ runs, we extract the substring after \texttt{Final:} (trim, strip quotes). For MCQ, we regex-match \texttt{[ABCD]}; for extractive QA, we normalize punctuation/whitespace (SQuAD-style). For summarization, we truncate to the requested budget (e.g., 200 words) and use the benchmark scorer verbatim.

\para{Compute matching and seeds.}
Unless otherwise specified, \emph{Thinking} uses $T_{\text{think}}=8192$ and \method{} uses $(k,N_{\text{TTT}})=(128,32)$ so that $T_{\text{think}}\approx 2N_{\text{TTT}}k$. We fix the random seed for span sampling and decoding across methods per run; results are averaged over one run per configuration (low variance in our setting).
% \footnote{If multiple seeds were used, add $n$ and report mean$\pm$sd.}

% Requires: \usepackage{booktabs}
\section{Score Dilution Evidence on Long Contexts}
\label{app:score-dilution}

\para{Motivation.}
Long-context failures could be a result of a multitude of reasons
and design choices.
Past literature in long-context modeling has primarily
focused on tuning positional encoding to improve
long-context abilities.
Here we present some evidence supporting our claim
that \emph{score dilution}
is one of the primary reasons for long-context failure.
We show that as the context grows,
attention mass on the target collapses,
and accuracy falls
even when rotary position embeddings (RoPE)
are present and the model is not changed otherwise.
We further show that qTTT counteracts this collapse
suggesting that our approach
actually counteracts score dilution in practice.

\para{Experimental setting (RoPE ablation).}
We evaluate Qwen3-4B on two tasks (Bank Transactions; OLMo Code Bugs) under three test-time regimes:
(1) \emph{Thinking-only} with a fixed thinking budget (4k or 8k tokens), (2) \emph{qTTT (ours)} with a brief
query-only adaptation while reusing the prefetched KV cache, and (3) a \emph{No-RoPE} ablation where we disable
rotary phase application to $Q/K$ at inference (identity rotation), keeping all weights, prompts, and budgets
unchanged and without any additional fine-tuning. This isolates the role of positional encoding while holding
training and data fixed.

\para{Attention-mass metric.}
For each decode step $t$, layer $\ell$, and head $h$, let $A^{(\ell,h)}_{t,\tau}$ denote the softmax attention
from the current query to context position $\tau$. Given a labeled set of target indices $\mathcal{T}$, we define
the \emph{attention mass} at step $t$ as $\sum_{\tau \in \mathcal{T}} A^{(\ell,h)}_{t,\tau}$, then average over
all layers and heads; for multi-token answers we average over their output steps. We report mean~$\pm$~std across
multiple runs.

\para{Findings.}
Tables~\ref{tab:bank-rope-ablation} and \ref{tab:olmo-rope-ablation} show that thinking-only accuracy and attention
mass both decay sharply with context length. Disabling RoPE accelerates this collapse (lower mass and accuracy),
but \emph{even with} RoPE the decline is substantial. In contrast, {qTTT sustains markedly higher attention
mass as context grows and correspondingly improves accuracy. These results support the view that score dilution, rather than
training-data scarcity alone, is the dominant failure mode in these settings.

\begin{table*}[t]
\centering
\small
\caption{Bank Transactions (Qwen3-4B): Accuracy (\%) and attention mass vs.\ context length with and without RoPE, and with qTTT.}
\begin{tabular}{lcccccc}
\toprule
\textbf{Context Tokens} &
\multicolumn{2}{c}{\textbf{Thinking (RoPE)}} &
\multicolumn{2}{c}{\textbf{Thinking (No-RoPE)}} &
\multicolumn{2}{c}{\textbf{qTTT (Ours)}}\\
\cmidrule(lr){2-3}\cmidrule(lr){4-5}\cmidrule(lr){6-7}
& \textbf{Acc} & \textbf{Mass} & \textbf{Acc} & \textbf{Mass} & \textbf{Acc} & \textbf{Mass}\\
\midrule
512   & 36.00 & $0.46 \pm 0.04$ & 34.00 & $0.44 \pm 0.04$ & 28.00 & $0.42 \pm 0.06$ \\
2{,}536 &  6.00 & $0.22 \pm 0.03$ &  5.00 & $0.20 \pm 0.02$ & 14.40 & $0.41 \pm 0.08$ \\
5{,}120 &  2.50 & $0.11 \pm 0.02$ &  0.80 & $0.03 \pm 0.01$ & 10.00 & $0.36 \pm 0.09$ \\
9{,}560 &  1.00 & $0.04 \pm 0.01$ &  0.50 & $0.01 \pm 0.00$ &  8.40 & $0.25 \pm 0.09$ \\
\bottomrule
\end{tabular}
\label{tab:bank-rope-ablation}
\end{table*}

\begin{table*}[t]
\centering
\small
\caption{OLMo Code Bugs (Qwen3-4B): Accuracy (\%) and attention mass vs.\ context length with and without RoPE, and with qTTT.}
\begin{tabular}{lcccccc}
\toprule
\textbf{Context Tokens} &
\multicolumn{2}{c}{\textbf{Thinking (RoPE)}} &
\multicolumn{2}{c}{\textbf{Thinking (No-RoPE)}} &
\multicolumn{2}{c}{\textbf{qTTT (Ours)}}\\
\cmidrule(lr){2-3}\cmidrule(lr){4-5}\cmidrule(lr){6-7}
& \textbf{Acc} & \textbf{Mass} & \textbf{Acc} & \textbf{Mass} & \textbf{Acc} & \textbf{Mass}\\
\midrule
512   & 50.00 & $0.64 \pm 0.05$ & 47.40 & $0.61 \pm 0.05$ & 45.70 & $0.58 \pm 0.06$ \\
2{,}050 & 21.60 & $0.38 \pm 0.07$ & 16.20 & $0.29 \pm 0.04$ & 41.60 & $0.51 \pm 0.08$ \\
7{,}450 & 17.20 & $0.26 \pm 0.06$ & 10.60 & $0.14 \pm 0.02$ & 28.00 & $0.42 \pm 0.09$ \\
10{,}000 & 10.00 & $0.12 \pm 0.03$ &  3.00 & $0.04 \pm 0.01$ & 20.20 & $0.35 \pm 0.09$ \\
\bottomrule
\end{tabular}
\label{tab:olmo-rope-ablation}
\end{table*}
\section{ZeroScrolls and LongBench-v2: All models and subsets.}
\label{app:all_results}

This appendix reports the complete breakdowns for all benchmarks, models, and inference settings. We compare three modes—vanilla in-context, chain-of-thought “Thinking”, and our test-time training method (\shortmethod{})—for Qwen3-1.7B/4B/8B across LongBench-v2 and ZeroScrolls. Unless otherwise noted, higher is better and bold indicates the best within each row/condition.

Figure~\ref{fig:longbenchv2-overview} shows a FLOP-matched overview of LongBench-v2 results across its six domains. The detailed per-domain numbers that underlie this figure appear in Table~\ref{tab:longbench_v2_results}.
Figure~\ref{fig:zeroscrolls-overview} summarizes the observed results on ZeroScrolls. The complete per-dataset numbers, including retrieval-heavy and summarization tasks, are provided in Table~\ref{tab:zeroscrolls_results}.

Tables~\ref{tab:qwen32_longbench} and \ref{tab:qwen32_zeroscrolls} shows results on the Qwen3-32B model. We see that similar trends hold across subsets of the two datasets, validating the efficacy of \shortmethod{} across model sizes.

% ---------- (a) LongBench-v2: 3 models ----------
\begin{figure*}[t]
    \centering
    \begin{minipage}{\textwidth}\centering
    \includegraphics[width=\linewidth]{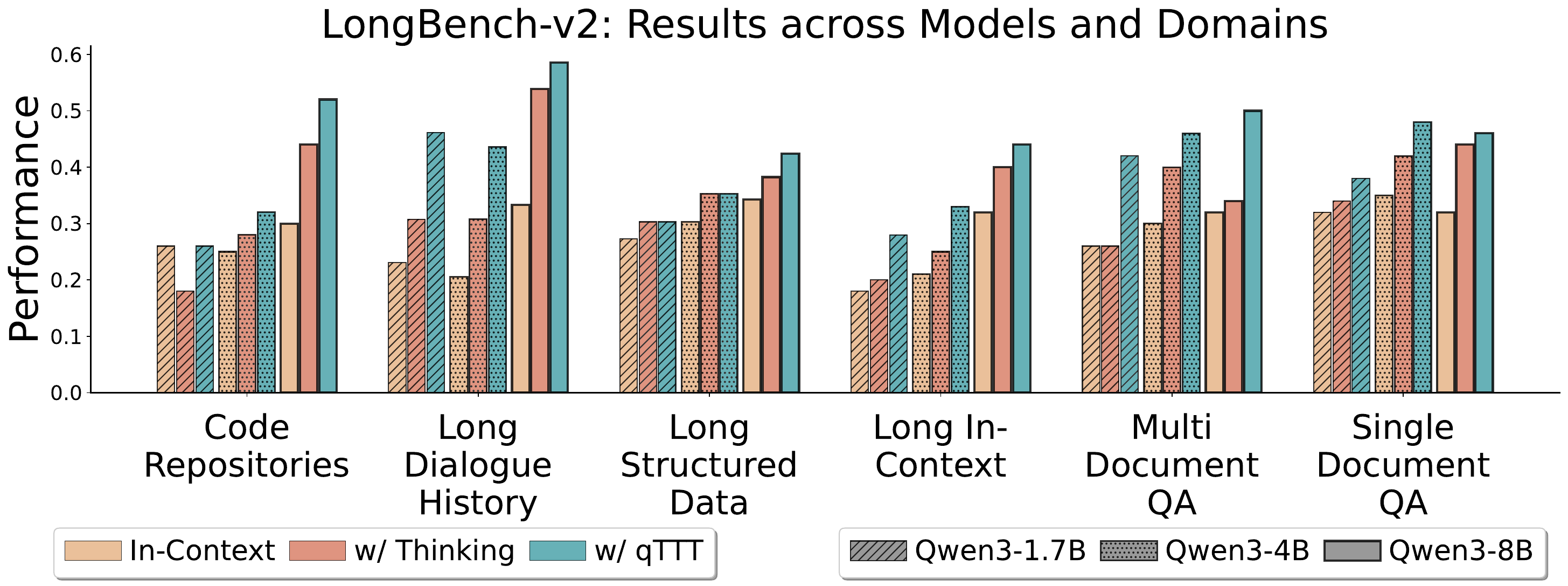}
    \end{minipage}
    \caption{FLOP-matched comparison on \textbf{LongBench-v2} \citep{bai2023longbench} across six domains for Qwen3-$1.7$B/$4$B/$8$B under vanilla in-context only, with thinking (CoT), and with test-time training (TTT). TTT consistently yields the best accuracy across domains and model sizes, with the largest gains on long-dialogue and document-QA tasks, and benefits growing with model size.}
    \label{fig:longbenchv2-overview}
\end{figure*}

% Table~\ref{tab:longbench_v2_results} lists the full LongBench-v2 results by domain and model size. Use this table to inspect how each inference mode scales across domains such as code repositories, long dialogue, and multi/single-document QA.

% LongBench-v2 %
\begin{table}[t]
\centering
\caption{Full \textbf{LongBench-v2} results for Qwen3-1.7B/4B/8B under In-context, Thinking, and \shortmethod{}. Scores follow benchmark-defined metrics; bold marks the best within each row/condition.}
\label{tab:longbench_v2_results}
\resizebox{\textwidth}{!}{%
\begin{tabular}{lccccccccc}
\toprule
\multirow{2}{*}{\textbf{}} & \multicolumn{3}{c}{\textbf{Qwen3-1.7B}} & \multicolumn{3}{c}{\textbf{Qwen3-4B}} & \multicolumn{3}{c}{\textbf{Qwen3-8B}} \\
\cmidrule(lr){2-4}\cmidrule(lr){5-7}\cmidrule(lr){8-10}
 & \textbf{In-context} & \textbf{Thinking} & \textbf{\shortmethod{}} & \textbf{In-context} & \textbf{Thinking} & \textbf{\shortmethod{}} & \textbf{In-context} & \textbf{Thinking} & \textbf{\shortmethod{}} \\
\midrule
Code Repositories      & \textbf{26.0} & 18.0 & \textbf{26.0} & 25.0 & 28.0 & \textbf{32.0} & 30.0 & 44.0 & \textbf{52.0} \\
Long Dialogue History  & 23.1 & 30.8 & \textbf{46.2} & 20.5 & 30.8 & \textbf{43.6} & 33.3 & 53.8 & \textbf{58.5} \\
Long Structured Data   & 27.3 & \textbf{30.3} & \textbf{30.3 }& 30.3 & \textbf{35.3} & \textbf{35.3} & 34.3 & 38.2 & \textbf{42.4} \\
Long In-Context        & 18.0 & 20.0 & \textbf{28.0} & 21.0 & 25.0 & \textbf{33.0} & 32.0 & 40.0 & \textbf{44.0} \\
Multi-Document QA      & 26.0 & 26.0 & \textbf{42.0 }& 30.0 & 40.0 & \textbf{46.0} & 32.0 & 34.0 & \textbf{50.0} \\
Single-Document QA     & 32.0 & 34.0 & \textbf{38.0 }& 35.0 & 42.0 & \textbf{48.0} & 32.0 & 44.0 & \textbf{46.0} \\
\midrule
\textbf{Average}       & 25.4 & 26.5 & \textbf{35.1} & 27.0 & 33.5 & \textbf{39.6} & 32.3 & 42.3 & \textbf{48.8} \\
\bottomrule
\end{tabular}%
}
\end{table}

% ---------- (a) Zero SCROLLS: 3 models ----------
\begin{figure*}[t]
    \centering
    \begin{minipage}{\textwidth}\centering
    \includegraphics[width=\linewidth]{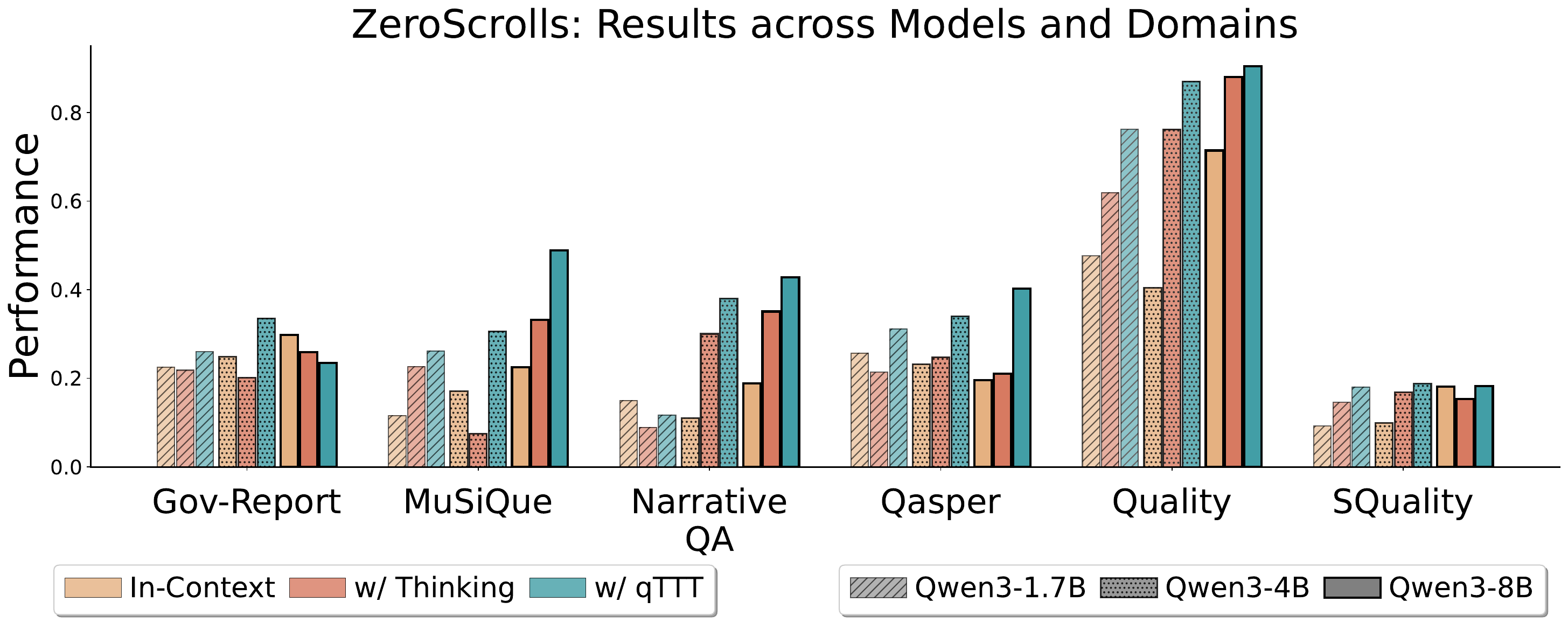}
    \end{minipage}
    \caption{FLOP-matched comparison on the \textbf{ZeroScrolls} benchmark \citep{shaham2023zeroscrolls} for Qwen3-1.7B/4B/8B under long contexts, with thinking (CoT), and with test-time training (TTT). TTT achieves the highest scores on nearly all datasets—especially on the retrieval-focused tasks, with a general increase with model size.
    % \rachit{enhance + group different dataset into buckets}
    }
    \label{fig:zeroscrolls-overview}
\end{figure*}

% Table~\ref{tab:zeroscrolls_results} presents the full ZeroScrolls results by dataset and model size. This table helps diagnose where each inference mode contributes most (e.g., multi-hop QA vs.\ summarization).

% ZeroScrolls %
\begin{table}[t]
\centering
\caption{Full \textbf{ZeroScrolls} results across eight datasets for Qwen3-1.7B/4B/8B under In-context, Thinking, and \shortmethod{}. Datasets span retrieval and summarization; bold marks the best within each row/condition (higher is better).}
\label{tab:zeroscrolls_results}
\resizebox{\textwidth}{!}{%
\begin{tabular}{lccccccccc}
\toprule
\multirow{2}{*}{\textbf{}} & \multicolumn{3}{c}{\textbf{Qwen3-1.7B}} & \multicolumn{3}{c}{\textbf{Qwen3-4B}} & \multicolumn{3}{c}{\textbf{Qwen3-8B}} \\
\cmidrule(lr){2-4}\cmidrule(lr){5-7}\cmidrule(lr){8-10}
 & \textbf{In-context} & \textbf{Thinking} & \textbf{\shortmethod{}} & \textbf{In-context} & \textbf{Thinking} & \textbf{\shortmethod{}} & \textbf{In-context} & \textbf{Thinking} & \textbf{\shortmethod{}} \\
\midrule
GovReport     & 22.5 & 21.8 & \textbf{26.0} & 24.9 & 20.2 & \textbf{33.5} & 22.0 & 17.8 & \textbf{29.8} \\
MuSiQue       & 11.6 & 22.6 & \textbf{26.2} & 17.1 & 7.5  & \textbf{30.5} & 22.5 & 43.9 & \textbf{48.9} \\
NarrativeQA   & 15.0 & 8.9  & \textbf{11.7} & 11.0 & 30.0 & \textbf{38.0} & 18.9 & 35.1 & \textbf{42.8} \\
QASPER        & 25.7 & 21.4 & \textbf{31.1} & 23.2 & 24.7 & \textbf{34.0} & 19.6 & 21.1 & \textbf{26.1} \\
QMSum         & 6.2  & 7.5  & \textbf{9.5}  & \textbf{10.9} & 7.7  & 8.6  &\textbf{ 9.8}  & 8.6  & 8.6  \\
QUALITY       & 47.6 & 61.9 & \textbf{76.2} & 40.5 & 76.2 & \textbf{87.0} & 71.4 & 90.5 & \textbf{94.5} \\
SQuALITY      & 9.2  & 14.6 & \textbf{18.0} & 9.9  & 16.8 & \textbf{18.7} & 18.1 & 15.3 & \textbf{18.3} \\
SummScreen-FD & 8.2  & 7.2  & \textbf{7.4}  & \textbf{9.9}  & 8.3  & \textbf{9.9}  & \textbf{10.4} & 7.9  & 7.9  \\
\midrule
\textbf{Average} & 18.3 & 20.7 & \textbf{25.8} & 18.4 & 23.9 & \textbf{32.5} & 24.1 & 30.0 & \textbf{34.6} \\
\bottomrule
\end{tabular}%
}
\end{table}

\begin{table}[ht]
\centering
\small
\caption{\textbf{Qwen3-32B on LongBench-v2.} Comparison of In-context, Thinking, and \shortmethod{}. These findings demonstrate that that the improvements with \shortmethod{} hold across model scales.}
\label{tab:qwen32_longbench}
\begin{tabular}{lccc}
\toprule
 & \textbf{In-context} & \textbf{Thinking} & \textbf{qTTT} \\
\midrule
Code Repositories & 36.00 & 61.00 & \textbf{74.00} \\
Long In-Context & 44.00 & 56.00 & \textbf{57.00} \\
Long Structured Data & 39.30 & 42.20 & \textbf{51.50} \\
Long Dialogue History & 47.10 & \textbf{77.90} & 75.50 \\
Multi Document QA & 35.00 & 41.00 & \textbf{56.00} \\
Single Document QA & 36.00 & 47.00 & \textbf{49.00} \\
\bottomrule
\end{tabular}
\end{table}

\begin{table}[ht]
\centering
\small
\caption{\textbf{Qwen3-32B on ZeroScrolls.} Comparison of In-context, Thinking, and \shortmethod{}. These findings demonstrate that that the improvements with \shortmethod{} hold across model scales.}
\label{tab:qwen32_zeroscrolls}
\begin{tabular}{lccc}
\toprule
 & \textbf{In-context} & \textbf{Thinking} & \textbf{qTTT} \\
\midrule
Gov Report & \textbf{26.70} & 24.80 & 26.00 \\
Musique & 28.90 & 54.90 & \textbf{59.20} \\
Narrative QA & 27.70 & 42.40 & \textbf{49.60} \\
Qasper & 24.10 & 35.00 & \textbf{42.40} \\
QMSum & \textbf{11.90} & 9.90 & 10.80 \\
\bottomrule
\end{tabular}
\end{table}

% Requires: \usepackage{booktabs}
\section{Additional Test-Time Scaling Baselines}
\label{app:tts-baselines}

\para{Baselines.}
We compare \textbf{Best-of-$N$ (BoN)} and \textbf{Beam Search} to our method under strict compute parity.
\emph{BoN / Self-Consistency (SC-$N$):} we run $N$ independent decodes, each with an equal share of the extra
reasoning budget, and select the final answer by majority vote (ties broken by sequence log-prob).
\emph{Beam-$k$:} we run left-to-right beam search of width $k$; to enforce parity with other test-time scaling,
the \emph{total} added ``thinking'' tokens across all beams is fixed.

\para{Design choices (strict matching).}
We match all methods to a fixed extra budget corresponding to $T_\text{think}=8192$ tokens beyond the vanilla decode.
SC-$N$ allocates \(\approx 8192/N\) tokens to each sample; Beam-$k$ allocates \(\approx 8192/k\) tokens per beam.
All results use the same prompt, output length (128 tokens); latencies are reported
separately in \S\ref{app:latency}. This protocol removes budget-induced confounders and isolates the effect of
test-time scaling itself.

\para{Conclusion.}
Across both LongBench-v2 and ZeroScrolls (Qwen3-4B), qTTT is competitive with or better than strictly
FLOP-matched BoN and Beam. SC-$N$ helps when single-run accuracy is already high (e.g., \textit{Single Document QA},
\textit{QUALITY}), but often degrades when the per-sample accuracy is below 50\%. Beam-$k$ provides only modest gains
under equal budgets due to correlated beams and imperfect ranking, and frequently trails qTTT.

\begin{table*}[t]
\centering
\small
\caption{LongBench-v2 (Qwen3-4B): Strict FLOP-matched test-time scaling. Numbers are accuracies (\%). SC-$N$ uses $8192/N$ tokens per sample; Beam-$k$ uses $8192/k$ tokens per beam.}
\resizebox{\textwidth}{!}{%
\begin{tabular}{lcccccccc}
\toprule
\textbf{Task} & \textbf{Thinking} & \textbf{qTTT} & \textbf{SC-8} & \textbf{SC-16} & \textbf{SC-32} & \textbf{Beam-8} & \textbf{Beam-16} & \textbf{Beam-32} \\
\midrule
Code Repositories    & 28.0 & 32.0 & 30.5 & 18.4 &  5.2 & 27.5 & 15.1 &  4.8 \\
Long In-Context      & 25.0 & 33.0 & 28.5 & 20.1 &  8.5 & 26.0 & 18.5 &  7.2 \\
Long Structured Data & 35.3 & 35.3 & 36.1 & 30.5 & 12.2 & 34.8 & 28.1 & 11.0 \\
Long Dialogue History& 30.8 & 43.6 & 34.2 & 31.0 & 15.5 & 29.5 & 25.2 & 12.0 \\
Multi Document QA    & 40.0 & 46.0 & 44.5 & 41.2 & 25.8 & 39.8 & 35.5 & 22.1 \\
Single Document QA   & 42.0 & 48.0 & 45.5 & 49.2 & 51.8 & 43.5 & 44.2 & 41.0 \\
\midrule
\textbf{Avg.}        & 33.5 & \textbf{39.7} & 36.6 & 31.7 & 19.8 & 33.5 & 27.8 & 16.4 \\
\bottomrule
\end{tabular}}
\label{tab:tts-longbench}
\end{table*}

\begin{table*}[t]
\centering
\small
\caption{ZeroScrolls (Qwen3-4B): Strict FLOP-matched test-time scaling. Numbers are accuracies (\%). SC-$N$ uses $8192/N$ tokens per sample; Beam-$k$ uses $8192/k$ tokens per beam.}
\resizebox{\textwidth}{!}{%
\begin{tabular}{lcccccccc}
\toprule
\textbf{Task} & \textbf{Thinking} & \textbf{qTTT} & \textbf{SC-8} & \textbf{SC-16} & \textbf{SC-32} & \textbf{Beam-8} & \textbf{Beam-16} & \textbf{Beam-32} \\
\midrule
gov report       & 20.2 & 33.5 & 24.5 & 15.2 &  2.1 & 22.8 & 12.5 &  1.8 \\
musique          &  7.5 & 30.5 & 18.2 & 12.5 &  4.5 & 14.5 &  9.8 &  3.2 \\
narrative qa     & 30.0 & 38.0 & 35.5 & 32.0 & 22.5 & 31.2 & 25.5 & 18.1 \\
qasper           & 24.7 & 34.0 & 28.5 & 22.1 & 10.5 & 26.5 & 19.5 &  9.2 \\
qmsum            &  7.7 &  8.6 &  9.2 &  5.1 &  0.8 &  8.5 &  4.5 &  0.7 \\
quality          & 76.2 & 87.0 & 82.5 & 85.1 & 84.5 & 78.5 & 76.2 & 70.1 \\
squality         & 16.8 & 18.7 & 17.5 & 19.2 & 20.5 & 17.1 & 17.5 & 17.8 \\
summ screen fd   &  8.3 &  9.9 &  9.5 &  6.5 &  1.2 &  8.8 &  5.5 &  1.1 \\
\midrule
\textbf{Avg.}    & 23.9 & \textbf{32.5} & 28.2 & 24.7 & 18.3 & 26.0 & 21.4 & 15.3 \\
\bottomrule
\end{tabular}}
\label{tab:tts-zeroscrolls}
\end{table*}
\section{Latency and Compute-Matched Measurements}
\label{app:latency}

\para{Setup.}
All latency numbers were measured on a single NVIDIA A100 GPU in standard inference mode.
We report wall-clock time in seconds (mean $\pm$ std) for three different context lengths.
For a given model size and context length,
we perform latency analysis based
on the amount of FLOPs, $F_{qTTT}$,
it takes to run $N_{qTTT}=32$
steps for $k=128$ on a single evaluation example.
We report the following metrics:

\begin{itemize}
    \item $N_{\text{think}}$: Number of thinking tokens
    that can be generated to match $F_{qTTT}$ FLOPs.

    \item $N_{\text{BoN}}$: Number of best-of-N trajectories
    that can be generated to match $F_{qTTT}$ FLOPs.

    \item $t_{\text{ICL}}$: Wall-clock time for a vanilla in-context pass on single example. This roughly corresponds to the prefill time.

    \item $t_{\text{think}}$: Wall-clock time to generate
    $N_{\text{think}}$ tokens, given a single example.

    \item $t_{\text{BoN}}$: The amount of time to compute best-of-N via self-consistency for $N_{\text{BoN}}$ trajectories given a single example.

    \item $t_{\text{qTTT}}$: The amount of time to perform $N_{qTTT}=32$ steps of qTTT steps with span length $k=128$
    for a single example.
\end{itemize}

\begin{table*}[t]
\centering
\small
\caption{Latency and wall-clock time comparisons given a fixed FLOP budget for Qwen3-1.7B.}
\begin{tabular}{lcccccc}
\toprule
\textbf{Context Length} & \textbf{$t_{\text{ICL}}$ (s)} & \textbf{$t_{\text{qTTT}}$ (s)} & \textbf{$t_{\text{think}}$ (s)} & \textbf{$t_{\text{BoN}}$ (s)} & \textbf{$N_{\text{think}}$} & \textbf{$N_{\text{BoN}}$}\\
\midrule
8,000     & $8.73 \pm 0.35$  & $16.92 \pm 0.68$ & $16.93 \pm 0.68$ & $16.05 \pm 0.64$ & 1{,}434 & 11 \\
32,000    & $34.93 \pm 1.40$ & $43.12 \pm 1.72$ & $43.11 \pm 1.72$ & $40.78 \pm 1.63$ & 358 & 3 \\
128,000   & $139.70 \pm 5.59$& $147.89 \pm 5.92$& $147.93 \pm 5.92$& $139.70 \pm 5.59$& 90 & 1 \\
\bottomrule
\end{tabular}
\label{tab:latency-1p7b}
\end{table*}

\begin{table*}[t]
\centering
\small
\caption{Latency and wall-clock time comparisons given a fixed FLOP budget for Qwen3-4B.}
\begin{tabular}{lcccccc}
\toprule
\textbf{Context Length} & \textbf{$t_{\text{ICL}}$ (s)} & \textbf{$t_{\text{qTTT}}$ (s)} & \textbf{$t_{\text{think}}$ (s)} & \textbf{$t_{\text{BoN}}$ (s)} & \textbf{$N_{\text{think}}$} & \textbf{$N_{\text{BoN}}$}\\
\midrule
8{,}000   & $14.61 \pm 0.58$  & $28.27 \pm 1.13$ & $28.26 \pm 1.13$ & $27.41 \pm 1.10$ & 1{,}365 & 11 \\
32{,}000  & $58.45 \pm 2.34$  & $72.11 \pm 2.88$ & $72.09 \pm 2.88$ & $68.69 \pm 2.75$ & 341 & 3 \\
128{,}000 & $233.81 \pm 9.35$ & $247.47 \pm 9.90$ & $247.41 \pm 9.90$ & $233.81 \pm 9.35$ & 85 & 1 \\
\bottomrule
\end{tabular}
\label{tab:latency-4b}
\end{table*}

\begin{table*}[t]
\centering
\small
\caption{Latency and wall-clock time comparisons given a fixed FLOP budget for Qwen3-8B.}
\begin{tabular}{lcccccc}
\toprule
\textbf{Context Length} & \textbf{$t_{\text{ICL}}$ (s)} & \textbf{$t_{\text{qTTT}}$ (s)} & \textbf{$t_{\text{think}}$ (s)} & \textbf{$t_{\text{BoN}}$ (s)} & \textbf{$N_{\text{think}}$} & \textbf{$N_{\text{BoN}}$}\\
\midrule
8{,}000   & $22.13 \pm 0.89$  & $42.61 \pm 1.70$ & $42.62 \pm 1.70$ & $41.33 \pm 1.65$ & 1{,}229 & 10 \\
32{,}000  & $88.53 \pm 3.54$  & $109.01 \pm 4.36$ & $109.00 \pm 4.36$ & $97.07 \pm 3.88$ & 307 & 2 \\
128{,}000 & $354.13 \pm 14.17$ & $374.61 \pm 14.98$ & $374.67 \pm 14.99$ & $354.13 \pm 14.17$ & 77 & 1 \\
\bottomrule
\end{tabular}
\label{tab:latency-8b}
\end{table*}

Tables~\ref{tab:latency-1p7b},~\ref{tab:latency-4b},~\ref{tab:latency-8b}
show the results of the measurements on
Qwen3-1.7B, 4B, and 8B, respectively.
We find that the wall-clock time
for all three test-time compute strategies---qTTT,
thinking, and best-of-N---is quite similar.
We also note that prefilling the KV cache,
which is approximately equal to $t_{\text{ICL}}$
dominates most of the decoding time,
especially for longer sequence lengths.
This motivates the frozen K/V attention weights
in our setup, without which the prefill
would need to be recomputed with every training step.

% \para{Implementation note.}
% qTTT’s adaptation runs outside CUDA-graph capture; decoding then proceeds with the same captured graphs
% since tensor shapes and kernels are unchanged. Because only $W_Q$ is updated and the KV cache is frozen,
% we retain a single prefill and production-style paged-KV behavior across all settings.

\end{document}